%% file: trunk 2/main.tex
\documentclass[5p,twocolumn]{elsarticle}

\input{packages}
\input{macros-common}
\input{macros}

\begin{document}

\begin{frontmatter}

\title{On Expansion and Contraction of DL-Lite Knowledge Bases}

\author[ocado]{Dmitriy Zheleznyakov}
\ead{d.zheleznyakov@ocado.com}

\author[oxf]{Evgeny Kharlamov}
\ead{evgeny.kharlamov@cs.ox.ac.uk}

\author[fub]{Werner Nutt}
\ead{Werner.Nutt@unibz.it}

\author[fub]{Diego Calvanese}
\ead{Diego.Calvanese@unibz.it}

\address[ocado]{Ocado Technology, Hatfield AL10 9UL, United Kingdom.}

\address[oxf]{University of Oxford, Department of Computer Science, Wolfson Building, Parks Road, OX1 3QD, Oxford, UK.}

\address[fub]{Free University of Bozen-Bolzano, Faculty of Computer Science,
              Piazza Domenicani 3,
              39100 Bolzano, Italy.}

\begin{abstract}
  Knowledge bases (KBs) are not static entities: new information constantly
  appears and some of the previous knowledge becomes obsolete. In order to
  reflect this evolution of knowledge, KBs should be expanded with the new
  knowledge and contracted from the obsolete one.  This problem is well-studied
  for propositional but much less for first-order KBs. In this work we
  investigate knowledge expansion and contraction for KBs expressed in \dllite,
  a family of description logics (DLs) that underlie the tractable fragment
  OWL\,2\,QL of the Web Ontology Language OWL\,2. We start with a novel
  knowledge evolution framework and natural postulates that evolution should
  respect, and compare our postulates to the well-established AGM
  postulates. We then review well-known model and formula-based approaches for
  expansion and contraction for propositional theories and show how they can be
  adapted to the case of \dllite. In particular, we show intrinsic limitations
  of model-based approaches: besides the fact that some of them do not respect
  the postulates we have established, they ignore the structural properties of
  KBs. This leads to undesired properties of evolution results: evolution of
  \dllite KBs cannot be captured in \dllite. Moreover, we show that well-known
  formula-based approaches are also not appropriate for \dllite expansion and
  contraction: they either have a high complexity of computation, or they
  produce logical theories that cannot be expressed in \dllite. Thus, we
  propose a novel formula-based approach that respects our principles and for
  which evolution is expressible in \dllite. For this approach we also propose
  polynomial time deterministic algorithms to compute evolution of \dllite KBs
  when evolution affects only factual data.
\end{abstract}


\end{frontmatter}

\input{10-introduction}

\input{20-preliminaries}

\input{31-evolution-framework}

\input{32-running_example}

\input{33-postulates}

\input{41-model-based_approaches}

\input{42-model-based-inexpressibility}

\input{51-formula-based_approaches}

\input{52-bold_semantics}

\input{60-abox_evolution}

\input{70-related_work.tex}

\input{80-conclusion}

\paragraph{Acknowledgements}

This work was partially funded by the EPSRC projects MaSI$^3$, DBOnto, ED$^3$,
by the UNIBZ projects PARCIS and TaDaQua, and by the ``European Region
Tyrol-South Tyrol-Trentino'' (EGTC) under the first call for basic research
projects within the Euregio Interregional Project Network IPN12
``Knowledge-Aware Operational Support'' (KAOS).

% \biboptions{sort&compress}
\biboptions{sort} % compress makes it difficult to search for refs in the pdf
\bibliographystyle{elsarticle-num}
\bibliography{refs}

%%%%%%%%%%%%%%%%%%%%%%%%%%%%%%%%%%%%%%%%%%%%%%%%%%%%%%%%%%%%%%%%%%%%%%%%%%%%%
% \newpage
%\input{appendix}

\end{document}

%% file: packages.tex
\usepackage{times}
\usepackage{graphicx}
\usepackage{xspace}
\usepackage[defblank]{paralist}
\usepackage{url}
\usepackage{comment}
\usepackage[scaled=.95]{newtxtt} % nicer mono font
\usepackage{latexsym}
\usepackage{amsmath,amssymb}
\usepackage{amsthm,thmtools,thm-restate}
\usepackage{hyperref}

\usepackage[algo2e,algoruled,linesnumbered,lined]{algorithm2e}
\SetNlSkip{0.5em}
\usepackage[only bigsqcap]{stmaryrd} % for \bigsqcap

%% file: macros-common.tex
\newcommand{\A}{\mathcal{A}} 
 \newcommand{\D}{\mathcal{D}}
\newcommand{\E}{\mathcal{E}} \newcommand{\F}{\mathcal{F}}
 
\newcommand{\I}{\ensuremath{\mathcal{I}}\xspace}
\newcommand{\J}{\mathcal{J}}
\newcommand{\K}{\ensuremath{\mathcal{K}}\xspace}
\renewcommand{\L}{\mathcal{L}}
\newcommand{\M}{\ensuremath{\mathcal{M}}\xspace}
\newcommand{\N}{\ensuremath{\mathcal{N}}\xspace}

\renewcommand{\S}{\mathcal{S}}
\newcommand{\T}{\ensuremath{\mathcal{T}}\xspace}
 \newcommand{\V}{\mathcal{V}}

\newcommand{\defterm}[1]{\mbox{\underline{\it\smash{#1}\vphantom{\lower.1ex\hbox
{x}}}}}

\newcommand{\set}[1]{\{#1\}}                      % set
\newcommand{\card}[1]{|#1|}                         % cardinality

\newcommand{\bigset}[1]{\Bigl\{#1\Bigr\}}

%% file: macros.tex
\newtheorem{theorem}{Theorem}[section]

\newtheorem{lemma}[theorem]{Lemma}
\newtheorem{proposition}[theorem]{Proposition}
\newtheorem{corollary}[theorem]{Corollary}

\theoremstyle{definition}

\declaretheoremstyle[
  spaceabove=6pt, spacebelow=6pt,
  postheadspace=1em,
  headfont=\normalfont\bfseries,
  notefont=\bfseries, notebraces={(}{)},
  bodyfont=\normalfont,
  qed=\qedboxfull,
  sibling=theorem
  ]{examplestyle}
\declaretheorem[style=examplestyle]{example}
\newcommand{\qedfull}{\hfill{\qedboxfull}
  \ifdim\lastskip<\medskipamount \removelastskip\penalty55\medskip\fi}

\newcommand{\qedboxfull}{\vrule height 5pt width 5pt depth 0pt}

\newcommand{\quotes}[1]{``#1''\xspace} % quotes
\newcommand{\eat}[1]{}  % commenting text out
\renewcommand{\phi}{\varphi}

\newcommand{\andauthor}{and\xspace}

\newcommand{\dom}{\Delta\xspace} % active domain
\newcommand{\Mod}{\textit{Mod}}    % all models
\newcommand{\dist}{\textit{dist}}  % distance function
\newcommand{\nptime}{\textsf{\small NP}\xspace}
\newcommand{\conptime}{\textsf{\small coNP}\xspace}

\newcommand{\tmodels}{\models_\calt\xspace}
\newcommand{\tequiv}{\equiv_\calt\xspace}

\newcommand{\EP}[1]{\textbf{E{#1}}\xspace}
\newcommand{\CP}[1]{\textbf{C{#1}}\xspace}

\newcommand{\argmin}{\mathop{\mathrm{\arg \min}}}
\newcommand{\loc}{\ensuremath{\mathbb{L}}\xspace}  % generic symbol for local semantics
\newcommand{\glob}{\ensuremath{\mathbb{G}}\xspace} % generic symbol for local semantics
\newcommand{\ats}{a}                   % generic symbol for atom based semantis
\newcommand{\sbs}{s}                   % generic symbol for symbol based semantics
\newcommand{\setin}{\incl}             % generic symbol for minimization wrt set inclusion
\newcommand{\cardin}{\#}           % generic symbol for minimization wrt cardinality

\newcommand{\BS}{bold semantics\xspace}

\newcommand{\upop}{\circ}       % an update operator
\newcommand{\controp}{\bullet}		% a contraction operator
\newcommand{\mdalop}{\upop_{M'}}  % modificated Dalal's revision operator
\newcommand{\sdalop}{\upop_S}    % stratified Dalal's revision operator
\newcommand{\upd}[3]{ #1 \upop_{#3} #2}    % an update operator
\newcommand{\lupd}[2]{\upd{#1}{#2}{\loc}}    % a local update operator
\newcommand{\WIDTIO}{\textit{WIDTIO}}
\newcommand{\CrossP}{\textit{CP}}

\newcommand{\contr}[3]{\ensuremath{#1 \controp_{#3} #2}\xspace}
\newcommand{\lcontr}[2]{\contr{#1}{#2}{\loc}}
\newcommand{\gcontr}[2]{\contr{#1}{#2}{\glob}}

\newcommand{\algfont}[1]{\ensuremath{\mathtt{#1}}\xspace}

\newcommand{\cl}{\textit{cl}\xspace}  % closure of a TBox
\newcommand{\tcl}{\textit{cl}_\calt\xspace}  % closure wrt T algorithm
\newcommand{\BE}{\algfont{BoldExpansion}}    % Bold expansion algorithm for KBs
\newcommand{\BC}{\algfont{BoldContraction}}  % Bold contraction algorithm for KBs

\newcommand{\FE}{\algfont{FastExpansion}}    % ABox Bold expansion algorithm
\newcommand{\FC}{\algfont{FastContraction}}    % ABox Bold contraction algorithm

\newcommand{\FUNCT}[1]{(\mathsf{funct}\; #1)}
\newcommand{\ISA}{\sqsubseteq}

\newcommand{\dllite}{\textit{DL-Lite}\xspace}
\newcommand{\dlcore}{\textit{DL-Lite}$_{\mathit{core}}$\xspace}
\newcommand{\dlf}{\textit{DL-Lite}$_{\cal F}$\xspace}

\newcommand{\dlfr}{\textit{DL-Lite}$_{\cal FR}$\xspace}

\newcommand{\incl}{\subseteq}

\newcommand{\calt}{\mathcal{T}}

\newcommand{\SIG}{\Sigma}          % Signature of a knowledge base

\newcommand{\KEX}{\K_{\textit{ex}}}

\newcommand{\concept}[1]{\mathtt{#1}}
\newcommand{\constant}[1]{\texttt{#1}}

\newcommand{\mary}{\constant{mary}}
\newcommand{\john}{\constant{john}}
\newcommand{\sam}{\constant{sam}}
\newcommand{\adam}{\constant{adam}}
\newcommand{\bob}{\constant{bob}}
\newcommand{\carl}{\constant{carl}}

\newcommand{\wife}{\concept{Wife}}

\newcommand{\empwife}{\concept{EmpWife}}
\newcommand{\priest}{\concept{Priest}}

\newcommand{\single}{\concept{Bachelor}}
\newcommand{\cleric}{\concept{Cleric}}
\newcommand{\minister}{\concept{Minister}}

\newcommand{\renter}{\concept{Renter}}

\newcommand{\hashusband}{\concept{HasHusband}}

\newcommand{\confa}{\mathit{CA}}

\newcommand{\eset}{\emptyset}              % emptyset
\newcommand{\col}{\colon}                  % colon

\newcommand{\domain}{\Delta}                  % domain of interpretation

%% file: 10-introduction.tex
%!TEX root = main.tex

\section{Introduction}
\label{sec:intro}

Description Logics (DLs) provide excellent mechanisms for representing
structured knowledge as \emph{knowledge bases} (KBs), and as such they constitute the foundations for the
various variants of OWL\,2, the standard ontology language of the Semantic
Web.\footnote{\url{http://www.w3.org/TR/owl2-overview/}}  KBs have
traditionally been used for modeling at the intensional level the static and
structural aspects of application domains~\cite{BoBr03}.
%%
%% The reasoning services that have been investigated for the currently used
%% DLs and implemented in state-of-the-art DL reasoners~\cite{HaMo03},
%% traditionally focus on so-called standard reasoning, both at the TBox level
%% (e.g., TBox satisfiability, concept satisfiability and subsumption wrt a
%% TBox) and at the ABox level (e.g., knowledge base satisfiability, instance
%% checking and retrieval, and more recently query
%% answering)~\cite{HaMo08,CaDL08}.
%%
Recently, however, the scope of KBs has broadened, and they are now used
also for providing support in the maintenance and evolution phase of
information systems~\cite{DBLP:conf/www/MartinBMMPSMSS07}.
Moreover, KBs are considered to be the premium
mechanism through which services operating in a Web context can be accessed,
both by human users and by other services~\cite{DBLP:conf/vldb/BerardiCGHM05,DBLP:journals/expert/McIlraithSZ01,DBLP:journals/ws/KuterSPNH05}.  Supporting all these activities,
makes it necessary to equip DL systems with additional kinds of inference servicies
that go beyond the traditional ones of satisfiability, subsumption, and query
answering provided by current DL inference engines.  A critical one, and
the subject of this paper, is that of \emph{KB evolution}~\cite{FMKPA08}.
%% \cite{DLPR09,LLMW06}

In KB evolution the task is to incorporate new knowledge $\N$ into an existing KB $\K$,
or to delete some obsolete knowledge $\N$ from $\K$,
in order to take into account changes that occur in the underlying domain of interest~\cite{KaMe91}.
The former evolution task is typically referred to as knowledge \emph{expansion} and the latter as \emph{contraction}.
In general, the new (resp., obsolete) knowledge is represented by a set
of formulas denoting those properties that should be true (resp., false) after the ontology has evolved.
In the case where the new knowledge interacts in an undesirable
way with the knowledge in the ontology, e.g., by causing the ontology or relevant parts of it to become unsatisfiable, the new knowledge cannot simply be added to the ontology.
Instead, suitable changes need to be made in the
ontology so as to avoid the undesirable interaction, e.g., by deleting parts of
the ontology that conflict with the new knowledge.
Different choices are
possible, corresponding to different semantics for knowledge evolution
\cite{AbGr85,Wins90,KaMe91,EiGo92,Flou06,QiDu09}.

The main two types of semantics that were proposed for
the case of propositional knowledge
are \emph{model-based}~\cite{Wins90} and \emph{formula-based}~\cite{EiGo92}.
In model-based semantics the idea is to resolve the undesirable interaction at the level of models of $\K$ and $\N$.
For example, in model-based expansion the result of evolution are those models of $\N$ that are minimally distant from the ones of~$\K$,
where
% the notion of distance depends on the application.
a suitable notion of distance needs to be chosen, possibly depending on the
application.
% An important feature of revision is that the distance is defined
% ``globally'' and depends on all the models of the old KB, in contrast to
% update, which works ``locally'' model by model.
%%
In formula-based semantics the idea is to do evolution at the level of
the deductive closure of the formulae from~$\K$ and~$\N$.
Since many (possibly counter-intuitive) semantics can be defined within the model or formula-based paradigm,
%In order to ensure that evolution results makes sense
a number of evolution \emph{postulates}~\cite{KaMe91,EiGo92} have been
proposed and they define natural properties a semantics should
respect.  It is thus common to verify for each evolution semantics
whether it satisfies the postulates.

%
% In the literature, two main types of ontology evolution have been considered:
% namely revision and update~\cite{KaMe91}.  In \emph{updating}, we assume that the
% real world has changed and the new knowledge $\N$ reflects the change.
% Therefore, we incorporate $\N$ into the old world description $\K$
% %% , which might be outdated and contradict $\N$.
% by updating every model of $\K$ with the new knowledge.
% %%
% In \emph{revision}, we assume that nothing happened with the real world, but we
% got to know that the new knowledge $\N$ is what is certainly true.
% %% The world's description $\K$ may happen to be erroneous or incomplete wrt
% %% $\N$ and we want to revise $\K$ with $\N$.
% Therefore, every model of a revised KB should satisfy $\N$ and should have
% minimal distance to the old KB, where the notion of distance depends on the
% application.  An important feature of revision is that the distance is defined
% ``globally'' and depends on all the models of the old KB, in contrast to
% update, which works ``locally'' model by model.
% %%
% Both update and revision have a precise formal grounding in terms of
% \emph{postulates}~\cite{KaMe91,EiGo92} and a number of update and revision
% operators were proposed in the literature~\cite{Wins90,EiGo92}.
% %% In \cite{QiDu09,Flou06} revision of DL knowledge bases was considered.

For the case of propositional knowledge,
there is a thorough understanding of
semantics as well as of computational properties
of both expansion and contraction.
The situation is however much less clear when it comes to
DL KBs, which are decidable first-order logic theories.
Differently from the propositional case, in general they
admit infinite sets of models
and infinite deductive closures.
Moreover, going from propositional letters to first-order predicates and
  interpretations, on the one hand calls for novel postulates underlying the
  semantics of evolution, and on the other hand broadens the spectrum of
  possibilities for defining such semantics.
A number of attempts have been made to adapt approaches for the
evolution of propositional knowledge to the case of DLs, cf.~\cite{Flou06,DLPR09,QiDu09,DBLP:journals/ai/LiuLMW11}
(see also the detailed discussion of related work in Section~\ref{sec:related-work}).
However, there is no thorough understanding of evolution from the foundational
point of view even for DLs with
the most favorable computational properties, such as the logics of the
\dllite~\cite{CDLLR07} and $\E\L$~\cite{BaaderBL:2005ijcai} families, which are
at the basis of two tractable fragments of OWL\,2.

In this work we address this problem and propose an exhaustive study of evolution for \dllite.
In particular, we address the problem considering three dimensions:
\begin{enumerate}
        \item knowledge evolution tasks: we study how knowledge can be \emph{expanded} or \emph{contracted};
        \item type of evolution semantics: we study \emph{model-based} and
          \emph{formula-based} semantics;
        \item evolution granularity:
        we study when evolution affects the \emph{TBox} (for terminological
        knowledge), or the \emph{ABox} (for assertional knowledge),
        or \emph{both} of them.
\end{enumerate}
We provide the following contributions:
\begin{itemize}
        \item We propose a knowledge expansion and contraction framework that accounts for TBox, ABox, and general KB evolution (Section~\ref{sec:formalism}).
        \item We propose natural evolution postulates and show how they are related to the well-known AGM postulates (Sections~\ref{sec:postulates} and~\ref{sec:connection-to-agm}).
        \item We show how one can rigorously extend propositional
          model-based evolution semantics to the first-order case,
          defining a five-dimensional space\footnote{The dimensions are (see Section~\ref{sec:mbas} for details):
         \begin{inparaenum}[\itshape (1)]
         \item ABox vs.\ TBox vs.\ general evolution;
         \item expansion vs.\ contraction;
         \item global vs.\ local;
         \item symbol vs.\ atom;
         \item set inclusion vs.\ cardinality.
         \end{inparaenum}} of possible options,
          comprising $3\cdot 2^4$ model-based evolution semantics for DLs
          that essentially
          include all previously proposed model-based approaches for DLs
          (Section~\ref{sec:mbas}).%
          \footnote{Note that our proposal for model-based semantics works for
           \emph{any} description logic and is not specific for \dllite.}
          For most of these semantics and the case of \dllite KBs we prove
          negative expressibility results: in general evolution of \dllite KBs
          cannot be expressed as a \dllite KB.
        % This is due to an intrinsic limitation of model-based semantics
        % in `preserving' the `structure' of the input knowledge bases,
        % that is, they `perform' evolution on the level of models
        % rather than the KB assertions and
        % and thus  move from KBS to models
        % This suggests that such semantics are not practical for applications.

        \item
        We investigate formula-based evolution for \dllite.  In
        particular, for known formula-based
        evolution approaches~\cite{EiGo92} we show intractability
                of computing evolution results for \dllite KBs.
        Moreover, we propose a practical though non-deterministic
        approach for general KB evolution, which turns out to
        become deterministic for ABox evolution; for both cases
        we develop practical algorithms (Section~\ref{sec:mbas}).
\end{itemize}

\paragraph{Delta from Previous Publications}

This article is based on our conference publication~\cite{CalvaneseKNZ:2010iswc}
while it significantly extends it in several important directions.
First,~\cite{CalvaneseKNZ:2010iswc} only considered knowledge
expansion, and thus all results on contraction in the current
article are new.
Second, in~\cite{CalvaneseKNZ:2010iswc} we had negative evolution
results for only a few out of the $3\cdot 2^3$ knowledge expansion
semantics, while here we show negative results for all but three of them.
Third, we strengthen the \conptime-hardness results for
the WIDTIO formula-based semantics presented
in~\cite{CalvaneseKNZ:2010iswc}.
Fourth, in contrast to the current
article,~\cite{CalvaneseKNZ:2010iswc} did not discuss how AGM
postulates as well as various formula-based semantics are related to
our postulates.
Finally, due to limited space, we did not include
in~\cite{CalvaneseKNZ:2010iswc} most of the proofs, while in this
article we included many more proofs, details, and explanations.

\paragraph{Structure of the Paper}

In Section~\ref{sec:Preliminaries}, we review the definition of \dllite.
In Section~\ref{sec:framework}, we present our evolution framework, including a comparison of our and the AGM postulates.
Then, in Section~\ref{sec:mbas}, we generalize model-based semantics for the evolution of propositional KBs to the first-order level
and investigate whether these semantics can be captured with the expressive means of \dllite.
In Section~\ref{sec:FormulaBasedApproaches}, we apply formula-based approaches to
  \dllite evolution and study their computational properties.
In Section~\ref{sec:AndABoxEvolution}, we show that for ABox evolution all formula-based semantics coincide for \dllite and
that this task is computationally feasible.
In Section~\ref{sec:related-work}, we discuss related work, and
in Section~\ref{sec:conclusion} we conclude the article and discuss future
work.

%% file: 20-preliminaries.tex
\section{Preliminaries}
\label{sec:Preliminaries}
%%%%%%%%%%%%%%%%%%%%%%%%%%%%%%%%%%%%%%%%%%%%%%%%%%%%%%%%%%%%%%%%%%%%%%

We now introduce some notions of Description Logics (DLs)
        that are needed for understanding the concepts
        used in this paper;
        more details can be found in~\cite{BCMNP03}.

A DL \emph{knowledge base} (KB) $\K= \T\cup\A$ is the union
  of two sets of assertions (or axioms),
  those representing the \emph{intensional-level} of the KB,
  that is, the general knowledge, and constituting the \emph{TBox} $\T$,
  and those providing information on the \emph{instance-level} of the KB, and
  constituting the \emph{ABox} $\A$.
In our work we consider the \dllite family \cite{CDLLR07,PLCD*08} of DLs,
which is at the basis of the tractable fragment
OWL\,2\,QL~\cite{w3c:owl2profiles} of OWL\,2~\cite{GrauHMPPS08,w3c:owl2}.

All logics of the \dllite family allow for constructing \emph{basic concepts}
$B$, (complex) \emph{concepts} $C$, and (complex) \emph{roles} $R$ according to
the following grammar:
\[
  B ::= A \mid  \exists R, \qquad
  C ::= B \mid  \lnot B, \qquad
  R ::= P \mid  P^-.
\]
where $A$ denotes an \emph{atomic concept} and $P$ an \emph{atomic role},
which are just names.

A \dlcore \emph{TBox} consists of \emph{concept inclusion assertions} of the
form
\[
  B \ISA C.
\]
\dlf extends \dlcore by allowing in a TBox also for
\emph{functionality assertions} of the form
\[
  \FUNCT{R}.
\]
\dlfr allows in addition for \emph{role inclusion assertions} of the form
\[
  R_1 \ISA R_2,
\]
in such a way that if $R_1 \ISA P_2$ or $R_1 \ISA P_2^-$
appears in a TBox $\T$, then neither $\FUNCT{P_2}$
nor $\FUNCT{P_2^-}$ appears in $\T$.
This syntactic restriction is necessary to keep reasoning in the logic
tractable \cite{PLCD*08}.

ABoxes in \dlcore, \dlf, and \dlfr consist of membership assertions of the form
\[
  B(a) \qquad \text{or} \qquad P(a,b),
\]
where $a$ and $b$ denote constants.

In the following, when we write \dllite without a subscript, we mean \emph{any}
of the three logics introduced above.

The semantics of \dllite KBs is given in the standard way, using first order
interpretations, all over the same infinite countable domain $\Delta$.
An \emph{interpretation} $\I$ is a (partial) function~$\cdot^\I$ that assigns to each
concept $C$ a subset $C^\I$ of $\Delta$, and to each role $R$ a binary relation
$R^\I$ over $\Delta$ in such a way that
\begin{align*}
  (P^-)^\I &~=~ \{(b,a)\mid (a,b) \in P^\I\},\\
  (\exists R)^\I &~=~ \{a\mid \exists b.(a,b) \in R^\I \},\\
  (\lnot B)^\I &~=~ \Delta\setminus B^\I.
\end{align*}
%%
%
% Moreover, $\I$ assigns to each constant $c$ an element $c^\I$ of $\Delta$ in
%   such a way that $c_1^\I\neq c_2^\I$ whenever $c_1\neq c_2$, i.e.,
%   we adopt the \emph{unique name assumption}.
%
We assume that $\Delta$ contains the constants and that $a^\I=a$, for each
constant $a$, i.e., we adopt \emph{standard names}.
Alternatively, we view an interpretation as a set of atoms and
say that $A(a)\in\I$ iff $a\in A^\I$, and that $P(a,b)\in\I$ iff
$(a,b)\in P^\I$.
An interpretation $\I$ is a \emph{model}
of a membership assertion $B(a)$ if $a\in B^\I$
and of $P(a,b)$ if $(a,b)\in P^\I$,
of an inclusion assertion $E_1\ISA E_2$ if $E_1^\I\incl E_2^\I$,
and of a functionality assertion $\FUNCT{R}$
if the relation $R^\I$ is a function, that is, for all $a,a_1,a_2\in\Delta$
we have that $\{(a,a_1),(a,a_2)\}\incl R^\I$ implies $a_1=a_2$.

As usual, we write $\I\models\alpha$ if $\I$ is a model of an assertion $\alpha$,
and $\I\models\K$ if $\I \models \alpha$ for each assertion $\alpha$ in $\K$.
We use $\Mod(\K)$ to denote the set of all models of $\K$.
A KB is \emph{satisfiable} if it has at least one model
and it is \emph{coherent}%
\footnote{Coherence is often called \emph{full satisfiability}.}
if for every atomic concept and atomic role $S$ occurring in $\K$
there is an $\I \in \Mod(\K)$ such that $S^\I \neq \eset$.
We use entailment, $\K\models\K'$, and equivalence, $\K\equiv\K'$, on KBs in
the standard sense.
Given a TBox $\T$, we say that an ABox $\A$ \emph{$\T$-entails} an ABox $\A'$,
denoted $\A\tmodels\A'$, if $\T\cup\A\models\A'$,
and that $\A$ is \emph{$\T$-equivalent} to $\A'$,
denoted $\A\tequiv\A'$, if $\A\tmodels\A'$ and $\A'\tmodels\A$.
The deductive \emph{closure of a TBox~$\T$}, denoted $\cl({\T})$, is the set of
all TBox assertions $\alpha$ such that $\T\models \alpha$.  Similarly, the
deductive \emph{closure of an ABox~$\A$} (w.r.t.\ a TBox~$\T$), denoted
$\cl_\T({\A})$, is the set of all ABox assertions $\alpha$ such that
$\T\cup\A\models \alpha$.
  It is easy to see that in \dllite, $\cl(\T)$ and $\cl_\T({\A})$ can be
computed in quadratic time in the size of $\T$ (and~$\A$).
In our work we assume that TBoxes and ABoxes are \emph{closed}, i.e., equal
to their deductive closure.

The \dllite family has nice computational properties,
  for example, KB satisfiability has polynomial-time complexity
  in the size of the TBox and logarithmic-space complexity in the size of the
  ABox~\cite{ACKZ09,PLCD*08}.

%% file: 31-evolution-framework.tex
%!TEX root = ./main.tex

\section{Knowledge Expansion and Contraction Framework}
\label{sec:framework}
%%%%%%%%%%%%%%%%%%%%%%%%%%%%%%%%%%%%%%

In this section, we first present our logical formalism of knowledge evolution,
then introduce our evolution postulates, and finally relate our postulates to
the well-known AGM postulates.

\subsection{Logical Formalism}
\label{sec:formalism}
%%%%%%%%%%%%%%%%%%%%%%%%%%%%%%%%%%%%%%

Consider a setting
 in which we have a knowledge base $\K=(\T,\A)$
 developed by knowledge engineers.
The KB $\K$ needs to be modified and
 a knowledge base $\N$ contains
 information about the modification.
Intuitively, we are interested in two scenarios that can be
described as follows:
\begin{itemize}
        \item
                \K is missing information captured in \N,
                % some fragments of \K may become outdated
                % or it became incomplete
                % and \M contains \emph{new} information that
                and this \emph{new} information \N
                should be incorporated in \K,
                that is, \K should be \emph{expanded} with \N.
        \item \K contains a modeling error, \N describes this error,
                and \K is to be \emph{contracted} by `extracting' \N from \K.
\end{itemize}

% Clearly it is often the case that there is no chance to say
%       how a given $\K$ is really related to the state of the real world or
%       which part of it is correct.
% %
% Indeed, consider a scenario in which we have a knowledge base $\K$
%       where we store knowledge from Web sources, say, online news papers
%       that we collected using RSS feeds or Web crawling.
% % Note that it is often the case that there is no chance to say
% %     how this information is related to the state of the real world or
% %     which part of the information is true.
% Consider the following two situations.
% In the first one,
%       a new portion $\N_e$ of information arrives
%       that may conflict with $\K$.
% Assume that information in $\N_e$ is newer than in $\K$,
%       or comes from a more trustworthy source, % or such,
%       so that, in case of conflicts, we give preference to $\N_e$ over $\K$.
% Thus, we need to incorporate the knowledge $\N_e$ into $\K$,
%       or \emph{expand} $\K$ with $\N_e$,
%       while resolving conflicts if they arise.
% In the second situation,
%       we do not receive any new information,
%       but we have learnt that a part $\N_c$ of old information
%       is false.
% Thus, we need to retract the knowledge $\N_c$ from $\K$,
%       or \emph{contract} $\K$ with $\N_c$.
% In either case
%       we want to study how $\K$ \emph{evolves}~\cite{FMKPA08} under $\N_e$ and $\N_c$.

More practically, we want to develop \emph{evolution operators}
        for both expansion and contraction of knowledge bases
        that take $\K$ and $\N$ as input and return,
        preferably in \emph{polynomial time}, a \dllite KB $\K'$
        that captures the evolution,
        and which we call \emph{the evolution of $\K$ under $\N$}.
As described above, we consider two evolution scenarios:
\begin{itemize}
\item ontology \emph{expansion},
        when $\N = \N_e$ represents the information that \emph{should hold}
        in $\K' = \K'_e$, and
\item ontology \emph{contraction},
        when $\N = \N_c$ defines the information that \emph{should not hold}
        in~$\K' = \K'_c$.
\end{itemize}

Our general assumption about the framework is the following.
We assume that both pieces of the new information, $\N_e$ and $\N_c$,
        are ``prepared'' to evolution,
        which means that $\N_e$ is coherent and $\N_c$ does not include tautologies.
Indeed,
        if $\N_e$ is not coherent,
        this means that the information in $\N_e$ is not true
        and thus, before incorporating it into $\K$,
        it is necessary to resolve issues with $\N_e$ itself.
If $\N_c$ contains tautological axioms,
        then it is clearly impossible to retract this knowledge from $\K$.

Additionally, apart from
\begin{enumerate}
\item \emph{KB evolution}, as described above,
\end{enumerate}
we distinguish two additional, special types of evolution:
\begin{enumerate} \setcounter{enumi}{1}
\item \emph{TBox evolution}, where $\N$ consists of TBox assertions only, and
\item \emph{ABox evolution}, which satisfies the following conditions:
        \begin{compactitem}%[\it (i)]
        \item the TBox of $\K'$ should be equivalent to $\T$,
        \item $\N$ consists of ABox assertions only, and
        \item in the case of expansion, $\T \cup \N$ is coherent.
        \end{compactitem}
Intuitively, ABox evolution corresponds to the case where the TBox $\T$ of $\K$
        is developed by domain specialists,
        does not contain wrong information, and
        should be preserved,
        while $\N$ is a collection of facts.
\end{enumerate}

We now illustrate these definitions on the following example.

%% file: 32-running_example.tex
%!TEX root = ./main.tex

%
\begin{example}[Running Example]

Consider a KB where the
structural knowledge is that
wives (concept $\wife$) are exactly those individuals who have husbands (role $\hashusband$)
and that some wives are employed (concept $\empwife$).
Bachelors (concept $\single$) cannot be husbands.
Priests (concept $\priest$) are clerics (concept $\cleric$) and clerics are bachelors.
Both clerics and wives are receivers of rent subsidies (concept $\renter$).
We also know that
$\adam$ and $\bob$ are priests,
$\mary$ is a wife who is employed and her husband is $\john$.
Also, $\carl$ is a catholic minister (concept $\minister$).

This knowledge can be expressed in \dlfr by the KB $\KEX$,
consisting of the following TBox $\T$ and ABox $\A$:
{\small
\begin{align*}
  \T &= \{~
  \begin{array}[t]{@{}ll@{}}
    \wife \ISA \exists \hashusband,
    & \exists \hashusband \ISA \wife ,\\
    \empwife \ISA \wife,
    & \single \ISA \lnot \exists \hashusband^-,\\
    \priest \ISA \cleric,
    & \cleric \ISA \single,\\
    \cleric \ISA \renter,
    & \wife \ISA \renter ~\}
  \end{array}\\
  \A &= \{~
  \begin{array}[t]{@{}l@{}}
    \priest(\adam),\quad
    \priest(\bob),\quad
    \minister(\carl),\\
    \empwife(\mary),\quad
    \hashusband(\mary,\john) ~\}
  \end{array}
\end{align*}}

In the expansion scenario the new information $\N_e$
states that John is now a bachelor, that is, $\single(\john)$, and
that catholic ministers are superiors of some religious orders and hence
clerics, that is, $\minister \ISA \cleric$.
Therefore:
\[
  \N_e = \{~
        \single(\john),\quad
        \minister \ISA \cleric
        ~\}.
\]
In the contraction scenario, due to an economic crisis,
rent subsidies got cancelled for % wives
priests,
%($\wife \ISA \lnot \renter$)
% and for clerics,
%($\cleric \ISA \lnot \renter$).
that is, $\N_c$ is
\[
        \N_c = \{~
%        \wife \ISA \renter,\
         \priest \ISA \renter
        ~\}.
\]
Later on in the paper we will discuss how to incorporate such new
knowledge $\N_e$ and $\N_c$ into the example KB $\KEX$.
\end{example}

%% file: 33-postulates.tex
%!TEX root = ./main.tex

\subsection{Postulates for Knowledge Base Evolution}
\label{sec:postulates}
%%%%%%%%%%%%%%%%%%%%%%%%%%%%%%%%%%%%%%
In the Semantic Web context,
update/revision and erasure/contraction~\cite{KaMe91,EiGo92},
        the classical understandings of ontology expansion and contraction, respectively,
        are too restrictive from the intuitive and formal perspective.
Indeed, on the one hand the `granularity' of knowledge changes when moving from propositional to Description Logics:
        the atomic statements of a DL, namely the ABox and TBox axioms,
        are more complex than the atoms of propositional logic.
On the other hand, a set of propositional formulas makes sense, intuitively,
        if it is satisfiable, while
                a KB can be satisfiable, but incoherent, that is,
        one or more concepts are necessarily empty.
Therefore, in the two following sections,
        we propose new postulates for expansion and contraction,
        to be adopted in the context of evolution on the Semantic Web.
%
%Let $\K$ be a (\dllite) knowledge base,
%	$\N$ new information expressed as a \dllite KB,
%	$\upop$ an expansion operator, and
%	$\K' = \upd{\K}{\N}{}$ a result of expansion of $\K$ with $\N$.
%But first we make the following assumption about our framework.
%We assume that our KB $\K$ ``makes sense'',
%	that is, we require it to be coherent.
%Moreover,
%	we assume that both pieces of the new information, $\N_e$ and $\N_c$
%	are ``prepared'' to evolution,
%	which means that $\N_e$ is coherent and $\N_c$ does not include tautologies.
%Indeed,
%	if $\N_e$ is not coherent,
%	then that means that the information in $\N_e$ is not true
%	and thus, before incorporating it into $\K$,
%	it is necessary to resolve issues with $\N$.
%If $\N_c$ contains tautological axioms,
%	then it is clearly impossible to retract this knowledge from $\K$.
%Having this set,
%	we are ready to turn to the postulates.

\paragraph*{Framework Postulates}
The first two postulates describe the basic requirements of our framework.
The first one is that evolution
        (both expansion and contraction)
        should preserve coherence:
\begin{quote}
        \EP{1}: Expansion should preserve the coherence of the KB,
        that is, if $\K$ is coherent, then so is $\K'_e$.\\[2ex]
        \CP{1}: Contraction should not add any extraneous knowledge,
        that is, $\K \models \K'_c$.
\end{quote}
Observe that \CP{1} does not say explicitly that contraction should preserve coherence;
        the latter, however, is implied.
The next postulate formalises the idea
        that expansion should incorporate new knowledge:
\begin{quote}
        \EP{2}: Expansion should entail all new knowledge,
        that is, $\K'_e \models \N_e$.
\end{quote}
Unfortunately,
        there is no obvious way to say
        what a corresponding contraction postulate should be.
Indeed,
        the most straightforward idea would be to say that $\K'_c \not\models \N_c$,
        that is, there should exist a model of $\K'_c$ that is not a model of $\N_c$.
This requirement, however, leads to undesirable consequences
        as shown in the following example.

\begin{example}
Consider the KB consisting of the two axioms $A \ISA B$ and $C \ISA D$.
Assume that we have learnt that both axioms are false
        and therefore the new information $\N_c$ consists of these two axioms.
Observe that it is the case for both $\K'_1 = \set{A \ISA B}$ and $\K'_2 = \set{C \ISA D}$
        that $\K'_i \not\models \N_c$.
However, intuitively, neither of them should be a result of contraction
        since either KB entails a piece of false information.
\end{example}

The example suggests that
        we need to make sure that $\K'_c$ does not entail \emph{each} axiom of $\N_c$.
There are two alternatives:
\begin{quote}
        \CP{2}: Contraction should not entail any piece of the new knowledge,
        i.e., $\K'_c \not\models \alpha$ for all $\alpha \in \N_c$.\\[2ex]
        \CP{2$'$}: Contraction should not entail the disjunction of the new knowledge,
        that is, $\K'_c \not\models \alpha_1 \lor\cdots\lor \alpha_n$,
        where $\N_c = \set{\alpha_1, \ldots, \alpha_n}$.
\end{quote}
Note that in general \CP{2} is strictly weaker than \CP{2$'$}
       when $\N_c$ contains more than one axiom.
That is, \CP{2$'$} entails \CP{2}, while the converse is not always the
case\footnote{The converse holds, however, for \dllite.  This is a direct
 consequence of Theorem~\ref{lem:no-disjunction-dllite}.}.
In our work we will focus rather on~\CP{2}.
Note also that
        most of our negative results
        hold already for contraction where $\N_c$
        is a singleton and thus, when these two postulates coincide.

\paragraph*{Basic Properties Postulates}
The next postulates define the basic property that
        evolution operators should satisfy;
        namely, it states when no changes should be applied to the KB:
%First, we state that if our KB already satisfies \EP{2}
%	(resp., \CP{2}),
%	then the operator should do nothing.
%
\begin{quote}
        \EP{3}: Expansion with old information should not affect the KB,
        that is, if $\K \models \N_e$, then $\K'_e \equiv \K$.\\[2ex]
        \CP{3}: Contraction with conflicting information should not affect the KB,
        that is,
        if $\K \not\models \alpha$ for each $\alpha \in \N_c$,
        then $\K'_c \equiv \K$.
%	\CP{3}: Contraction with irrelevant information should not affect the KB,
%	that is,
%	if $\K' \not\models \alpha$ for each $\alpha \in \N$,
%	then $\K' \equiv \K$.
%	\dima{check this postulate and weaken it if required}
\end{quote}
Observe that we can also define the postulate \CP{3$'$},
        which is an alternative to \CP{3}, but based on \CP{2$'$}.

%Also,
%	if the new knowledge does not bring any conflicts,
%	than expansion should not discard any of the old information:
%
%\begin{quote}
%	\EP{4}: When expanding with non-conflicting knowledge,
%	the new knowledge should be just added to the KB,
%	that is,
%	if $\K \cup \N_e$ is satisfiable, then $\K'_e \equiv \K \cup \N_e$.
%	\dima{Weaken this postulate}
%\end{quote}
%
%At the same time,
%	contraction should only get rid of the wrong knowledge and
%	not add anything new:
%\begin{quote}
%	v
%\end{quote}

The next two postulates define the preciseness of evolution:
\begin{quote}
        \EP{4}: The union of $\N_{2e}$ with the expansion of $\K$ with
        $\N_{1e}$ implies the expansion of $\K$ with $\N_{1e} \cup \N_{2e}$.\\[2ex]
        \CP{4}: The union of $\N_c$ with the contraction of $\K$ with $\N_c$
        implies $\K$.
        % \EP{4}: Expansion of $\K$ with $\N_{1e}$, united with $\N_{2e}$,
        % implies expansion of $\K$ with $\N_{1e} \cup \N_{2e}$.\\\\
        % \CP{4}: Contraction of $\K$ with $\N_c$, united with $\N_c$,
        % implies $\K$.
\end{quote}

\paragraph*{Principle Postulates}
The final two postulates
        represent evolution principles that are widely accepted in the literature.
The first one is the principle of \emph{irrelevance of syntax}:
\begin{quote}
        \EP{5}: Expansion should not depend on the syntactical representation of knowledge,
        that is,
        if $\K_1 \equiv \K_2$ and $\N_{1e} \equiv \N_{2e}$,
        then $\K'_{1e} \equiv \K'_{2e}$.\\[2ex]
        \CP{5}: Contraction should not depend on the syntactical representation of knowledge,
        that is,
        if $\K_1 \equiv \K_2$ and $\N_{1c} \equiv \N_{2c}$,
        then $\K'_{1c} \equiv \K'_{2c}$.
\end{quote}
Also, the so-called principle of \textit{minimal change}
        is widely accepted in the literature~\cite{EiGo92,KaMe91,Wins90}:
\begin{quote}
        The change to $\K$ should be minimal,
        that is, $\K'_e$ and $\K'_c$ are minimally different from~$\K$.
\end{quote}
However, there is no general agreement on how to define
        this minimality and the current belief is that
        there is no general notion of minimality
        that will ``do the right thing'' under all circumstances~\cite{Wins90}.
In this work we will follow this belief and we will incorporate
        some suitable notion of minimality into each evolution semantics we
        introduce.

\subsection{Connection to AGM Postulates}
     \label{sec:connection-to-agm}
%%%%%%%%%%%%%%%%%%%%%%%%%%%%%%%%%%%%%%
In this section we discuss the connection between our postulates and
        the AGM postulates~of Alchourr{\'o}n et al.~\cite{AGM85}.
The AGM approach has strongly influenced the formulation of
postulates by~Katsuno and Mendelzon in~%
        \cite{KaMe91}.
Given a (propositional) knowledge base $\psi$ and a sentence $\mu$,
        then $\psi \circ \mu$ denotes the \emph{revision} of $\psi$ by $\mu$;
        that is, the new knowledge base obtained by adding new knowledge $\mu$
        to the old knowledge base $\psi$.
        The following are the AGM postulates for revision:
\begin{enumerate}
\item[\bf (P+1)] $\psi \upop \mu$ implies $\mu$.
% \item[\bf (P+2)] If $\psi \land \mu$ is satisfiable,
%       then $\psi \upop \mu \leftrightarrow \psi \land \mu$.
\item[\bf (P+2)] If $\psi \land \mu$ is satisfiable,
        then $\psi \upop \mu \equiv \psi \land \mu$.
\item[\bf (P+3)] If $\mu$ is satisfiable,
        then $\psi \upop \mu$ is also satisfiable.
% \item[\bf (P+4)] If $\psi_1 \leftrightarrow \psi_2$ and $\mu_1 \leftrightarrow \mu_2$,
%       then $\psi_1 \upop \mu_1 \leftrightarrow \psi_2 \upop \mu_2$.
\item[\bf (P+4)] If $\psi_1 \equiv \psi_2$ and $\mu_1 \equiv \mu_2$,
        then $\psi_1 \upop \mu_1 \equiv \psi_2 \upop \mu_2$.
\item[\bf (P+5)] $(\psi \upop \mu) \land \phi$ implies $\psi \upop (\mu \land \phi)$.
\item[\bf (P+6)] If $(\psi \upop \mu) \land \phi$ is satisfiable,
        then $\psi \upop (\mu \land \phi)$ implies $(\psi \upop \mu) \land \phi$.
\end{enumerate}
Observe that \textbf{(P+1)} corresponds to our postulate \EP{2},
        \textbf{(P+3)} to \EP{1},
        \textbf{(P+4)} to \EP{5}, and
        \textbf{(P+5)} to \EP{4}.
Note that we do not have a postulate corresponding to \textbf{(P+6)},
        and instead of one corresponding to \textbf{(P+2)},
        we have the strictly \textit{weaker} postulate \EP{3}.%
\footnote{Compare \EP{3} with \textbf{(U2)} in~\cite{KaMe91}.}
The reason is that \textbf{(P+2)} and \textbf{(P+6)}
        reflect the view of Alchourr{\'o}n et al.\ on the Principle of Minimal Change;
        we, however, would like to study a broader class of operators
        than the one considered by Alchourr{\'o}n et al.,
        so we did not adapt \textbf{(P+6)} and weakened~\textbf{(P+2)}.

Now we turn to the AGM postulates for contraction.
Given a (propositional) knowledge base $\psi$ and a sentence $\mu$,
        then $\psi \controp \mu$ denotes the contraction of $\psi$ by $\mu$.
The following are the AGM postulates for contraction:
\begin{enumerate}
\item[\bf (P--1)] $\psi$ implies $\psi \controp \mu$.
\item[\bf (P--2)] If $\psi$ does not imply $\mu$,
        then $\psi \controp \mu$ is equivalent to $\psi$.
\item[\bf (P--3)] If $\mu$ is not a tautology,
        then $\psi \controp \mu$ does not imply $\mu$.
% \item[\bf (P--4)] If $\psi_1 \leftrightarrow \psi_2$ and $\mu_1 \leftrightarrow \mu_2$,
%       then $\psi_1 \controp \mu_1 \leftrightarrow \psi_2 \controp \mu_2$.
 \item[\bf (P--4)] If $\psi_1 \equiv \psi_2$ and $\mu_1 \equiv \mu_2$,
        then $\psi_1 \controp \mu_1 \equiv \psi_2 \controp \mu_2$.
\item[\bf (P--5)] $(\psi \controp \mu) \land \mu$ implies $\psi$.
\end{enumerate}
Observe that~\textbf{(P--1)} corresponds to \CP{1},
        \textbf{(P--3)} to \CP{2},
        \textbf{(P--4)} to~\CP{5}, and
        \textbf{(P--5)} to \CP{4}.
Similarly to the case of expansion,
        we have substituted \textbf{(P--2)} with the weaker postulate \CP{3}.%
        \footnote{Compare \CP{3} with \textbf{(E2)} in~\cite{KaMe91}.}

%% file: 41-model-based_approaches.tex
%!TEX root = ./main.tex

\section{Model-based Approaches to Evolution}
\label{sec:mbas}
%%%%%%%%%%%%%%%%%%%%%%%%%%%%%%%%%%%%%%%%%%%%%%%%%%%%%%%%%%%%%%%%%%%%%%

Among the candidate semantics for evolution operators
proposed in the literature we study first the \emph{model-based approaches}
(MBAs)~\cite{Wins90,GiSm87,PLCD*08,QiDu09}.
%
%  They can be divided into two groups, \emph{model-based approaches} (MBAs)
%  and \emph{formula-based approaches} (FBAs).
%
The section is organized as follows.
First, we define MBAs along several dimensions.
Then, we show negative results for MBAs in the context of \dllite.
Finally, we discuss conceptual problems of MBAs.

\subsection{Definition of Model-based Approaches to Evolution}
\label{sec:mba expansion}
%%%%%%%%%%%%%%%%%%%%%%%%%%%%%%%%%%%%%%%%%%%%%%%%%%%%%%%%%%%%%%%%%%%%%

We first define model-based expansion and then proceed to contraction.

%\adz{In this section we discuss the approaches to both expansion and contraction.
%To define them formally,
%	we naturally extend the notion of conformation
%	in the following way.
%Let \N be a set of DL assertions.
%We say that an interpretation \I
%	\emph{positively} (resp.\ \emph{negatively}) \emph{conforms to} \N,
%	denoted $\I \pconf \N$ (resp.\ $\I \nconf \N$)
%	if
%	$\I \in \Mod(\bigwedge_{F \in \N} F)$
%	(resp.\ $\I \notin \Mod(\bigvee_{F \in \N} F)$).
%}
%
\paragraph*{Model-based Expansion}

In MBAs, the result of the expansion of a KB $\K$
        w.r.t.\ new knowledge~$\N$ is a
        set $\upd{\K}{\N}{}$ of models.
The general idea of MBAs is
  to choose as the result of evolution some
  models of $\N$ depending on their distance to the models of $\K$.
Katsuno and Mendelzon~\cite{KaMe91} considered two ways of choosing these models of $\N$.

In the first one, which we call \emph{local}, the idea
is to go over all models $\I$ of $\K$ and
for each $\I$ to take those models $\J$ of $\N$ that are minimally distant from
$\I$.
Formally,
\[
  \upd{\K}{\N}{\loc}=\bigcup_{\I \in \Mod(\K)}\argmin_{\J \in \M} \dist(\I, \J),
\]
where
\begin{inparaenum}[\itshape (i)]
\item $\dist(\cdot,\cdot)$ is a function that varies from approach to approach,
  and whose range is a partially ordered domain,
\item $\argmin$ stands for the \emph{argument of the minimum,} that is, in our
  case, the set of interpretations $\J$ for which the value of $\dist(\I, \J)$
  reaches a minimum given $\I$, and
\item $\M$ is equal to
  \begin{inparablank}
  \item $\Mod(\N)$ in the case of KB evolution, or
  \item $\Mod(\T \cup \N)$ in the case of ABox evolution.
  \end{inparablank}
\end{inparaenum}
The distance function $\dist$ commonly takes as values either numbers or
subsets of some fixed set, and the minimum is defined according to the partial
order over its range.

In  the second way, called \emph{global},
the idea is to choose those models of $\N$ that are minimally distant from the
entire set of models of $\K$.
Formally,
\begin{equation}
  \label{eq:global-evolution}
  \upd{\K}{\N}{\glob} ~=~ \argmin_{\J \in \M} \dist(\Mod(\K), \J),
\end{equation}
where $\dist(\Mod(\K), \J) = \min_{\I \in \Mod(\K)} \dist(\I, \J)$ and
        $\M$ is as in the previous case.
Note that the minimum need not be unique, e.g., if distances are measured in terms of sets.
Then the distance between $\Mod(\K)$ and $\J$ is the set of all minimal distances $\dist(\I, \J)$
between elements $\I$ of $\Mod(\K)$ and $\J$.

To get a better intuition of local semantics, consider Figure~\ref{fig:MBS},
        which depicts two models $\I_1$ and $\I_2$ of $\K$,
        and four interpretations $\J_1,\ldots,\J_4$ that satisfy \N.
The distance between $\I_i$ and $\J_j$ is represented by the shape of the line connecting them:
        solid lines correspond to minimal distances, and dashed ones
        to distances that are not minimal.
In this case, $\J_1$ is in $\lupd{\K}{\N}$,
        because it is minimally distant from $\I_1$,
        and $\J_3$ and $\J_4$ are in $\lupd{\K}{\N}$,
        because they are minimally distant from $\I_2$.

%%%%%%%%%%%%%%%%%%%%%%
\begin{figure}[t!]%Model-based evolution semantics: example
\centering
\includegraphics[scale=0.5]{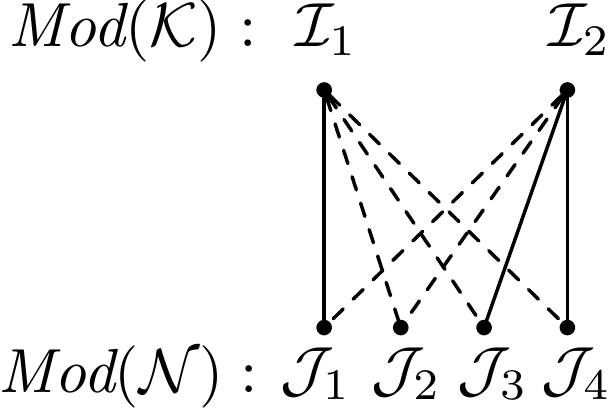}
\caption[Example of model-based evolution semantics]%
{Model-based evolution semantics: example}
\label{fig:MBS}
\end{figure}
%%%%%%%%%%%%%%%%%%%%%%

\paragraph*{Model-based Contraction}
In the literature, contraction in the DL setting
        received much less attention than expansion.
The general view on contraction,
        which originates from the ideas of contraction in propositional logic
        (see Section~\ref{sec:connection-to-agm}),
        is that the resulting set of models can be divided into two parts:
        first, the models of the original KB $\K$ (cf.~\CP{4}),
        and second, interpretations that falsify the axioms of $\N$ (cf.~\CP{2})
        and that are minimally distant from the models of $\K$.
Following this view,
        we define local and global model-based contraction operators as follows:
\begin{align*}
        \lcontr{\K}{\N} &~=~
                \Mod(\K)\ \cup
                        \bigcup_{\phi \in \N} \;
                        \bigcup_{\I \in \Mod(\K)\;}
                        \argmin_{\J \in \M_{\lnot\phi}} \dist(\I, \J), \\
        \gcontr{\K}{\N} &~=~
                \Mod(\K)\ \cup
                        \bigcup_{\phi \in \N} \;
                        \argmin_{\J \in \M_{\lnot\phi}} \dist(\Mod(\K), \J),
\end{align*}
        where $\M_{\lnot\phi}$ is equal to
        \begin{inparaenum}[\it (i)]
        \item $\set{\J \mid \J \not\models \phi}$ in the case of KB evolution, or
        \item $\set{\J \mid \J \in \Mod(\T) \text{ and } \J \not\models \phi}$
        in the case of ABox evolution.
        \end{inparaenum}
Observe that the second part of each definition,
        which builds the part of $\contr{\K}{\N}{}$ that falsifies $\N$,
        can be defined differently
        (e.g., the condition in the definition of $\M_{\lnot\phi}$
        could be $\J \not\models \bigvee_{\phi \in \N} \phi$,
        which corresponds to \CP{2$'$})
        or in a more general way (e.g., see~\cite{Zheleznyakov:2013thesis}).
We argue, however, that most model-based contraction operators
        satisfying our postulates
        coincide with one of our operators in the case when $|\N| = 1$,
        and since all of our negative results hold already for this case,
        they also apply to these other definitions.

\paragraph*{Three-dimensional Space of MBAs}

The classical MBAs have been developed for propositional theories.
In this context, an interpretation can be identified with the set of propositional atoms
that it makes true, and two distance functions have been introduced.
They are respectively based on the symmetric difference and on the cardinality
of the symmetric difference of interpretations, namely
\begin{equation}
  \label{eq:distances}
  \dist_\setin(\I, \J) = \I \ominus \J
  \quad\text{and}\quad
  \dist_\cardin(\I, \J) = \card{\I \ominus \J},
\end{equation}
where the symmetric difference $\I\ominus \J$ of two sets $\I$ and $\J$ is
defined as
$\I\ominus \J = (\I\setminus\J) \cup (\J\setminus\I)$.
Distances under $\dist_\setin$ are sets and are compared by set inclusion,
that is,
$\dist_\setin(\I_1, \J_1) \leq \dist_\setin(\I_2, \J_2)$
iff $\dist_\setin(\I_1, \J_1) \incl \dist_\setin(\I_2, \J_2)$.
Distances under $\dist_{\cardin}$ are natural numbers and are compared
in the standard way.

One can extend these distances to DL interpretations in two different ways.
One way is to consider interpretations $\I$, $\J$ as sets of \emph{atoms}.
Then $\I\ominus\J$ is again a set of atoms and we can define distances
as in Equation~\eqref{eq:distances}.
We denote these distances as
$\dist^\ats_\setin(\I, \J)$ and $\dist^\ats_\cardin(\I, \J)$, respectively.
While in the propositional case distances are always finite,
note that this may not be the case for DL
interpretations that are infinite.
Another way is to define distances at the level of the concept and role
\emph{symbols} in the signature $\SIG$ underlying the interpretations:
%\vspace{-5pt}
\begin{align*}
  \dist^\sbs_\setin(\I, \J) &= \set{S \in \SIG \mid S^\I \neq S^{\J}},
  \quad \text{and} \\
  \dist^\sbs_\cardin(\I, \J) &= \card{\set{S\in \SIG \mid S^\I \neq S^{\J}}}.
%\vspace{-5pt}
\end{align*}

%%%%%%%%%%%%%%%%%%%%%%%%%%%%%%%%%%%%%%%%%%
\begin{figure}[t!]%Model-based evolution semantics: example
\centering
\includegraphics[scale=0.5]{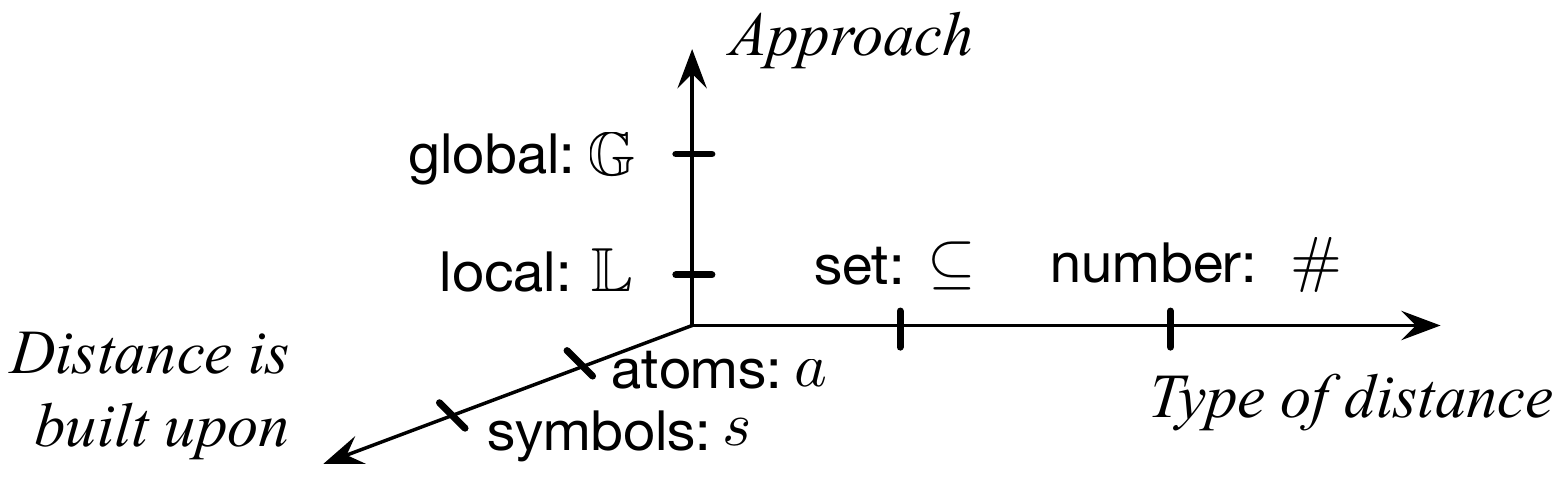}
\caption{Three-dimensional space of model-based evolution semantics}
\label{fig:3d space}
\end{figure}
%%%%%%%%%%%%%%%%%%%%%%%%%%%%%%%%%%%%%%%%%%

Summing up across the different possibilities,
we have three dimensions,
which give eight possibilities to define
a semantics of evolution according to MBAs by choosing (as depicted in Figure~\ref{fig:3d space}):
\begin{compactenum}
        \item the \emph{local} or the \emph{global} approach,
        \item \emph{atoms} or \emph{symbols} for defining distances, and
        \item \emph{set inclusion} or \emph{cardinality} to compare
              symmetric differences.
\end{compactenum}

We denote each of these eight possibilities by a combination of three symbols,
indicating the choice in each dimension.
By $\loc$ we denote local and by $\glob$ global semantics.
We attach the superscripts $\ats$ or $\sbs$ to indicate whether
distances are defined in terms of atoms or symbols, respectively.
And we use the subscripts $\setin$ or $\cardin$ to indicate whether
distances are compared in terms of set inclusion or cardinality, respectively.
For example, $\loc^\ats_\cardin$ denotes the local semantics
where the distances are expressed in terms of cardinality of sets of atoms.

Considering that in the propositional case a distinction between atom and symbol-based
semantics is meaningless,
we can also use our notation, without superscripts, to identify MBAs in that setting.
Interestingly, the  two classical local MBAs proposed
by Winslett~\cite{Wins90} and
Forbus~\cite{Forb89} correspond, respectively, to
$\loc_\setin$, and $\loc_\cardin$,
while the one by Borgida~\cite{Borg85} is a variant of $\loc_\setin$.
The two classical global MBAs proposed
by Satoh~\cite{Sato88} and Dalal~\cite{Dala88}
correspond respectively to
$\glob_\setin$, and~$\glob_\cardin$.

Next, we show that these semantics satisfy the evolution postulates
        defined in Section~\ref{sec:postulates}.

\begin{proposition}
\label{prop:model-based-operators-satisfy-postulates}
For $\mathbb{X} \in \set{\glob, \loc}$, $y \in \set{s, a}$ and $z \in \set{{\setin}, \cardin}$,
\begin{compactitem}
\item the expansion operator $\upd{}{}{\mathbb{X}^y_z}$ satisfies {\EP{1}}--\,{\EP{5}};
\item the contraction operator $\contr{}{}{\mathbb{X}^y_z}$ satisfies {\CP{1}}--\,{\CP{5}}.
\end{compactitem}
\end{proposition}

% \begin{proposition}
% \label{prop:model-based-operators-satisfy-postulates}
% The expansion (respectively,\ contraction) operators
%       $\upd{}{}{\mathbb{X}^y_z}$ (respectively, $\contr{}{}{\mathbb{X}^y_z}$),
%       where $\mathbb{X} \in \set{\glob, \loc}$, $y \in \set{s, a}$,
%       and $z \in \set{\setin, \cardin}$,
%       satisfy {\EP{1}}--\,{\EP{5}} (resp., {\CP{1}}--\,{\CP{5}}).
% \end{proposition}

\begin{proof}
The claim for \EP{1}, \EP{2}, \EP{5}, \CP{1}, \CP{2}, and \CP{5} follows directly
        from the definitions of the operators.
\EP{3} follows from the observation that
        if $\K \models \N$,
        then $\M$ in the definition of the operators coincides with $\Mod(\K)$,
        and thus each model of $\M$ is minimally distant from itself.
\CP{3} follows from the observation that
        if $\K \not\models \alpha$ for each $\alpha \in \N$,
        then the models of $\M_{\lnot\phi}$ minimally distant
        from some model $\I$ of $\K$ (resp., from $\K$)
        are exactly those models of $\K$ that falsify some $\phi$.
Regarding \EP{4},
        the claim is trivial if $(\upd{\K}{\N_1}{\mathbb{L}^y_z}) \cup \N_2$ is not satisfiable.
        If it is satisfiable,
        then observe that if $\I \in \Mod(\K)$
        and $\J_0 \in \argmin_{\J \in \Mod(\N_1)} \dist(\I, \J) \cap \Mod(\N_2)$,
        then $\J_0 \in \argmin_{\J \in \Mod(\N_1 \cup \N_2)} \dist(\I, \J)$.
The proofs for the case of $\mathbb{G}^y_z$ and the case of ABox expansion are similar.
Finally, \CP{4} follows from the following observation:
        if $\J_0 \in (\contr{\K}{\N}{\mathbb{L}^y_z}) \cup \N$,
        Then $\J_0 \in (\Mod(\K) \cup \M') \cap \Mod(\N)$,
        where $\M' = \argmin_{\J \in \M_{\lnot\phi}} \dist(\I, \J)$
        for some model $\I$ of~$\K$ and some $\phi \in \N$.
From $\J_0 \in \Mod(\N)$ we conclude that $\J_0 \models \phi$ and consequently
        $\J_0 \in \Mod(\K) \setminus \M'$, which proves the claim.
The proof for the case of $\mathbb{G}^y_z$ is similar.
\end{proof}

Under each of our eight semantics,
        expansion results in a set of interpretations.
In the propositional case,
        each set of interpretations over finitely many symbols
can be captured by %% a suitable formula, that is,
a formula whose models are exactly those interpretations.
In the case of DLs,
        this is not necessarily the case,
        since on the one hand,
        a KB might have infinitely many infinite models
        and, on the other hand, logics may lack some connectives
        like disjunction or negation.
Thus, a natural problem arising in the case of DLs
        is the \emph{expressibility} problem.

Let $\D$ be a DL and $\mathbb{M}$ one of the eight MBAs introduced above.  We
say that $\D$ is \emph{closed under expansion for $\mathbb{M}$} (or that
\emph{expansion w.r.t.\ $\mathbb{M}$ is expressible in $\D$}), if for all KBs
$\K$ and $\N$ written in~$\D$, there is a KB $\K'$ also written in $\D$ such
that $\Mod(\K') = \upd{\K}{\N}{\mathbb{M}}$.
Analogously, we say that $\D$ is \emph{closed under contraction} for
$\mathbb{M}$ (or that \emph{contraction w.r.t.\ $\mathbb{M}$ is expressible in
 $\D$}), if for all KBs $\K$ and $\N$ written in~$\D$, there is a KB $\K'$ also
written in $\D$ such that $\Mod(\K') = \contr{\K}{\N}{\mathbb{M}}$.
We study now whether \dlfr is closed under evolution w.r.t.\ the various
semantics.

%% file: 42-model-based-inexpressibility.tex
%!TEX root = ./main.tex

%%%%%%%%%%%%%%%%%%%%%%%%%%%%%%%%%%%%%%%%%%%%%%%%%%%%%%%%%%%%%%%%%%
%\subsection{Global Model-based Approaches}
\subsection{Inexpressibility of Model-based Approaches}
%%%%%%%%%%%%%%%%%%%%%%%%%%%%%%%%%%%%%%%%%%%%%%%%%%%%%%%%%%%%%%%%%%
%
We show now
        that both expansion and contraction,
        w.r.t.\ the introduced semantics are inexpressible in \dlfr.
Moreover, all our inexpressibility results
        hold already for TBox evolution, and for five of the eight considered
        semantics we show it for ABox evolution.

The key observation underlying these results is that, on the one hand,
the principle of minimal change often introduces implicit disjunction
        in the resulting KB.
On the other hand,
\dlfr can be embedded into a slight extension of Horn logic~\cite{CaKN07}
and therefore does not allow one to express genuine disjunction.
Technically, this can be expressed by saying that
every \dlfr KB that entails a disjunction of \dlfr assertions
entails one of the disjuncts.
The theorem below gives a contrapositive formulation of this statement.
Although \dlfr does not have a disjunction operator,
by abuse of notation we write $\J\models \phi\lor\psi$ as a shorthand for
\quotes{$\J\models\phi$ or $\J\models \psi$},
for \dlfr assertions $\phi$ and $\psi$.

%%%%%%%%%%%%%%%%
\begin{theorem}
        \label{lem:no-disjunction-dllite}
        Let $\M$ be a set of interpretations.
        Suppose there are \dllite assertions
        $\phi$,~$\psi$ such that
        \begin{compactenum}
                \item $\J\models\phi\lor \psi$
                      for every $\J\in\M$, and
                \item there are $\J_\phi$, $\J_\psi\in\M$ such that
                          $\J_\phi \not\models \phi$
                          and
                          $\J_\psi \not\models \psi$.
        \end{compactenum}
        Then, there is no \dlfr KB $\K$ such that $\M=\Mod(\K)$.
% \dima{It probably makes sense to make $\phi$ and $\psi$
%       arbitrary FO formulae
%       since this way it is easier to proof
%       some of our inexpressibility results.}
\end{theorem}

\begin{proof}
We prove the theorem by contradiction.
Assume there exists a \dlfr KB $\K$ such that
for every model $\J$ of $\K$ we have $\J\models\phi\lor \psi$,
but $\K\not\models\phi$ and $\K\not\models\psi$.

We distinguish the two cases
\begin{inparaenum}[\itshape (1)]
\item $\phi$ and $\psi$ are membership assertions, and
\item $\phi$ is an arbitrary assertion while $\psi$ is an inclusion or functionality assertions.
\end{inparaenum}

\medskip

\noindent
\textit{Case 1.}
This part of the proof relies on a result by Calvanese et al.~\cite{CDLLR07}
who showed that for every satisfiable \dlfr KB $\K$
there exists a model $\I_\K$, the \emph{canonical\/} model of $\K$,
that can be homomorphically mapped to every other model of $\K$.
Formally, for every model $\J$ there is a mapping $h\col\dom\to\dom$ such that
\begin{inparaenum}[\itshape (i)]
\item $h(a) = a$ for every constant $a$ appearing in $\K$,
\item $h(A^{\I_\K}) \subseteq A^\J$ for every atomic concept $A$, and
\item $h(P^{\I_\K}) \subseteq P^\J$ for every atomic role $P$.
\end{inparaenum}
In essence, the canonical model is constructed by chasing the ABox of $\K$
with the positive inclusion assertions in the TBox of $\K$,
that is, the inclusion assertions without negation sign.
Intuitively, the homomorphism $h$ exists because every model $\J$ of $\K$ satisfies these assertions,
and therefore all atoms introduced by the chase into $\I_\K$ have a corresponding atom in $\J$.
(Technically, there is a slight difference between our definition of interpretations
and the one in~\cite{CDLLR07}, as we assume that all interpretations share the same domain,
while domains can be arbitrary non-empty sets in~\cite{CDLLR07}.
The argument in~\cite{CDLLR07}, however, can be carried over in a straightforward way to our setting.)

Now, for the canonical model $\I_\K$ of $\K$ we have $\I_\K\models \phi\lor\psi$.
Then one of $\phi$ and $\psi$ is satisfied by $\I_\K$, say $\phi$.
However, since $\I_\K$ is canonical,
$\phi$ is also satisfied by every other model $\J$ of~$\K$,
due to the existence of a homomorphism from $\I_\K$ to $\J$.
For example, if $\phi = A(a)$, then $\I_\K\models A(a)$ implies $a\in A^{\I_\K}$,
which implies $a = h(a) \in h(A^{\I_\K}) \subseteq A^\J$, that is, $\J\models A(a)$.
For other kinds assertions, a similar argument applies.
This contradicts the assumption that there exists a $\J_\phi$ that falsifies~$\phi$.

\medskip

\noindent
\textit{Case 2.} Let $\K = \T \cup \A$.
The argument for this case will be based on the fact that
the disjoint union of a model of $\K$ and a model of the $\T$
is again a model of $\K$,
while the disjoint union of a counterexample for $\phi$ and
a counterexample for $\psi$ is a counterexample for both.
In order to formalise this idea % make this idea work,
we need some notation and simple facts as a preparation.

Given two interpretations $\I_1$, $\I_2$, their \emph{union} $\I_1\cup \I_2$
is the interpretation defined by $A^{\I_1\cup \I_2} = A^{\I_1}\cup A^{\I_2}$
for every primitive concept $A$ and
$P^{\I_1\cup \I_2} = P^{\I_1}\cup P^{\I_2}$ for every primitive role $P$.
From the definition it follows also for all concepts of the form $B = \exists R$,
where $R$ is one of $P$ or $P^{-}$,
that $B^{\I_1\cup \I_2} = B^{\I_1}\cup B^{\I_2}$.

We define the \emph{support set} of $\I$ %, denoted as $\supp(\I)$,
as the set of constants that occur in the interpretation of some atomic concept or role under $\I$.
If $\I_1$, $\I_2$ have disjoint support sets, we denote their union also as $\I_1\uplus \I_2$
and speak of a \emph{disjoint union}.

Let $\alpha$ be an inclusion or functionality assertion, let $\beta$ be a membership assertion,
and let $\I_1$, $\I_2$ be interpretations with disjoint support.
Then the following statements are straightforward to check:
\begin{enumerate}[\itshape (i)]
\item
    $\I_1\uplus \I_2\models\alpha$ iff
    $\I_1\models\alpha$ and $\I_2\models\alpha$;
\item
    if the support set of $\I_2$ is disjoint from the set of constants of $\beta$, then
    $\I_1\uplus \I_2\models\beta$ iff
    $\I_1\models\beta$;
\item
     $\I_1\uplus \I_2$ is a model of $\T$ iff
    $\I_1$ and $\I_2$ are both models of~$\T$;
\item
    if the support set of $\I_2$ is disjoint from the set of constants of $\K$, then
    $\I_1\uplus \I_2$ is a model of $\K$ iff
    $\I_1$ is a model of $\K$ and $\I_2$ is a model of $\T$.
\end{enumerate}
Here, \textit{(iii)} follows from \textit{(i)}, and \textit{(iv)} from  \textit{(ii)} and  \textit{(iii)}.
The assumption about the disjoint support sets is needed
in \textit{(i)}, \textit{(iii)} and \textit{(iv)} to guarantee that
negative inclusion assertions and functionality assertions continue to hold in $\I_1\uplus \I_2$.
In addition, the assumptions about the disjointness of constant sets and the support set of~$\I_2$
in \textit{(ii)} and \textit{(iv) } are needed
to guarantee that $\I_2$ has no influence on the satisfaction of membership assertions.

Next, we introduce a technique to create disjoint variants of interpretations
by moving their support sets with injective functions.
If $f\col\dom\to\dom$ is an injective mapping,
then the \emph{image of $\I$} under $f$ is the interpretation $\I^f$
satisfying $A^{\I^f} = f(A^\I)$ for every atomic concept $A$ and
$P^{\I^f} = f(P^\I)$ for every atomic role $P$.
If $K$ is a set of constants, we say that \emph{$f$ respects $K$} if $f(a)=a$ for every constant $a\in K$.

Let $\alpha$ be an inclusion or functionality assertion, $\beta$ a membership
assertion, and $\I$ an interpretation.
Then the following statements are straightforward to check:
\begin{enumerate}[\itshape (i)] \setcounter{enumi}{4}
\item
     $\I \models \alpha$ if and only if $\I^f \models \alpha$;
\item
     if $f$ respects the constants occurring in $\beta$, then
     $\I\models\beta$ iff $\I^f\models\beta$;
\item
     $\I$ is a model of $\T$ iff $\I^f$ is a model of~$\T$;
\item
     if $f$ respects the constants of $\K$, then
     $\I$ is a model of $\K$ iff $\I^f$ is a model of $\K$.
\end{enumerate}
Note that the injectivity assumption is needed for $\I^f$ to satisfy negative inclusion assertions
and functionality assertions if $\I$~does.

Now, suppose that $\phi$ is an arbitrary assertion and that
$\psi$ is an inclusion or functionality assertion.
Moreover, let $\J_\phi$, $\J_\psi$ be models of $\K$ such that $\J_\phi\not\models\phi$ and $\J_\psi\not\models\psi$.
To create disjoint variants of these interpretations,
we choose injective mappings $f$, $g\col\dom\to\dom$ such that
\begin{inparablank} %[\itshape (i)]
\item $f$ respects the constants of $\K$ and $\phi$, and
\item $f(\dom) \cap g(\dom) = \emptyset$.
\end{inparablank}
Clearly, such mappings always exist.
From the facts about images of interpretations, we conclude that
\begin{enumerate}
\item %[a)]
$\J_\phi^f$ is a model of $\K$ and
%\item %[b)]
$\J_\phi^f\not\models\phi$;
\item %[c)]
$\J_\psi^g$ is a model of $\T$ and
%\item %[d)]
$\J_\psi^g\not\models\psi$;
\item %[e)]
$\J_\phi^f$ and $\J_\psi^g$ have disjoint support sets.
\end{enumerate}
Hence, for $\J=\J_\phi^f\uplus \J_\psi^g$ we have that
$\J$ is a model of $\K$ and
$\J$ falsifies both $\phi$ and $\psi$.
% that is, $\J\not\phi\lor\psi$.
This contradicts the assumption that every model of $\K$
satisfies one of $\phi$ or $\psi$.
\end{proof}

\subsubsection{KB Evolution}
%%%%%%%%%%%%%%%%%%%%%%%%%%%%%%%%%%%%%%%%%%

In this part we show that \dlfr is not closed under TBox evolution (both
expansion and contraction) for \emph{any} of the introduced MBAs.
We start with the following example that illustrates the issue.

%%%%%%% ex: Global MBAs are not expressible in \dllite
\input{examples/for_section_31/ex-rev_operators_inexpressible}

We now proceed to our first inexpressibility result, for KB expansion.

%%%%%%%%% th: All global MBAs are not expressible in dllite

%\begin{restatable}{theorem}{TBoxInexpressibility}
\begin{theorem}
  \label{thm:tbox-mbs-inexpressible-dllite}
\dlfr is not closed under KB expansion
%\dlfr is neither closed under KB expansion
%	nor KB contraction
  for $\mathbb{X}^y_z$,
  where $\mathbb{X} \in \set{\glob, \loc}$,
  $y \in \set{\sbs, \ats}$, and $z \in \set{\setin, \cardin}$.
Moreover, this holds already
        when both the initial KB and the new information are written in \dlcore
        and the new information consists of a single TBox axiom.
\end{theorem}
%\end{restatable}

\begin{proof}%[Proof (Sketch)]
The main idea of the proof is that
        evolution changes models in such a way
        that capturing them all would require to have a disjunction,
        which is impossible by Theorem~\ref{lem:no-disjunction-dllite}.
We generalize the idea of Example~\ref{ex:rev-operator-inexpr}
where the inexpressibility of TBox expansion w.r.t.\ $\glob^\sbs_\cardin$
has already been shown.
%	and the other inexpressibility results can be shown similarly,
%	so we skip the details, that can be found in the appendix.

To show inexpressibility of expansion w.r.t.\ all eight semantics,
we consider the same fragment of our running example:
%$\K = \K_T\cup \K_A$ and $\N = \N_T \cup \N_A$, where
\begin{align*}
  \K_T &~=~ \set{\empwife \ISA \wife,\ \empwife \ISA \renter},\\
  \K_A &~=~ \set{\empwife(\mary)},\\
  \N_T &~=~ \set{\wife \ISA \lnot\renter},
\end{align*}
and $\K = \K_T\cup \K_A$.

We first consider expansion under global semantics.
With an argument as in Example~\ref{ex:rev-operator-inexpr}, one verifies that
there are models $\I$ of $\K$ where
only $\mary$ is both a $\wife$ and a $\renter$,
and such models can be turned into models $\J$ of $\N_T$
by either dropping $\mary$ from the set of wives or from the set of renters.
For these models we have
\begin{itemize}
\item $\dist^a_\setin(\I, \J) = \set{\wife(\mary)}$ or \\[.5ex]
      $\dist^a_\setin(\I, \J) = \set{\renter(\mary)}$;
\item $\dist^a_\cardin(\I, \J) = 1$;
\item $\dist^s_\setin(\I, \J) = \set{\wife}$ or
  $\dist^s_\setin(\I, \J) = \set{\renter}$;
\item $\dist^s_\cardin(\I, \J) = 1$.
\end{itemize}
Under each of the concerned semantics, these distances are minimal because
smaller distances could only be 0 or the empty set, respectively,
and interpretations with cardinality distance 0 or empty set-difference are identical.
Hence, for every model $\J \in \upd{\K}{\N_T}{\glob^y_z}$ there is a model $\I\in\K$ that differs
from $\J$ only in the interpretation of one concept, either $\wife$ or $\renter$.
It follows that
\begin{inparaenum}[\itshape (1)]
\item each such $\J$ either satisfies $\wife(\mary)$ or $\renter(\mary)$ and
\item there are $\J$ that satisfy one of the two assertions, but not the other.
\end{inparaenum}
Thus, by Theorem~\ref{lem:no-disjunction-dllite},
for none of the global semantics it is possible to express expansion in \dlfr.

Next, we turn to local semantics.
The arguments used here are a slight variant of the ones above,
taking into account the difference between the two kinds of semantics.
We start with some $\I\in\Mod(\K)$.
In such a model, $\mary$ for sure is both a $\wife$ and a $\renter$,
but there may be further individuals that are instances of both of these concepts.
Such an $\I$ can be turned into a model $\J\in\N_T$ by dropping for each individual
$o\in \wife^\I \cap \renter^\I$ either the atom $\wife(o)$ or the atom $\renter(o)$.
%\begin{itemize}

%\item
With respect to the atom-based distances $\dist^a_\setin$ and $\dist^a_\cardin$,
each $\J$ obtained in this way has minimal distance to~$\I$.
Moreover, these are the only models of $\N_T$ with minimal distance to $\I$
because further changes would increase the difference set and therefore the difference count.

%\item
With respect to the symbol-based distances $\dist^s_\setin$ and $\dist^s_\cardin$,
a $\J$ obtained in this way is only minimal if the dropped atoms all have the same symbol.
In this case we have again
$\dist^s_\setin(\I, \J) = \set{\wife}$ or $\dist^s_\setin(\I, \J) = \set{\renter}$ and,
correspondingly,
$\dist^s_\cardin(\I, \J) = 1$.
There are, however, further models of $\J\in\N_T$ with the same minimal distance to $\I$,
namely those that in $\J$ interpret individuals as instances of $\wife$
(or $\renter$, respectively)
that in $\I$ were neither instances of $\wife$ nor of $\renter$.
%\end{itemize}

In summary, since $\I$ was chosen arbitrarily, we have seen again that
\begin{inparaenum}[\itshape (1)]
\item each $\J\in \upd{\K}{\N_T}{\loc^y_z}$ either satisfies $\wife(\mary)$ or $\renter(\mary)$ and
\item there are $\J$ that satisfy one of the two assertions, but not the other.
\end{inparaenum}
So, the conditions of Theorem~\ref{lem:no-disjunction-dllite} are satisfied and thus
for none of the local semantics it is possible to express expansion in \dlfr.
\end{proof}

We now proceed to our second inexpressibility result, for KB contraction.

%\begin{restatable}{theorem}{TBoxInexpressibility}
\begin{theorem}
  \label{thm:tbox-mbs-contraction-inexpressible-dllite}
\dlfr is not closed under KB contraction
  for $\mathbb{X}^y_z$,
  where $\mathbb{X} \in \set{\glob, \loc}$,
  $y \in \set{\sbs, \ats}$, and $z \in \set{\setin, \cardin}$.
Moreover, this holds already
        when both the initial KB and the new information are written in \dlcore
        and the new information consists of a single TBox axiom.
\end{theorem}
%\end{restatable}

\begin{proof}
To show inexpressibility of contraction
% under all eight semantics,
we consider another fragment of our running example:
%$\K = \K_T\cup \K_A$ and $\N = \N_T \cup \N_A$, where
\begin{align*}
  \K &~=~ \set{\priest \ISA \cleric,\ \cleric \ISA \renter},  \\
  \N_T &~=~ \set{\priest \ISA \renter}.
\end{align*}
%and $\K = \K_T$.

We first consider local semantics.
To obtain $\contr{\K}{\N_T}{}$,
we have to add to $\Mod(\K)$ all interpretations $\J$
that falsify $\priest \ISA \renter$ and that
are minimally distant to some model $\I$ of~$\K$,
where distance is measured by one of the four measures
defining the local semantics.

Let $\I$ be a model of $\K$.
Then $\priest^\I \incl \cleric^\I$, $\cleric^\I \incl \renter^\I$, and hence
$\priest^\I \incl \renter^\I$.

There are, in principle, two ways to minimally change $\I$ in such a way that
$\priest \ISA \renter$ is no more satisfied.
For one, we can add an individual $o\in\domain\setminus\renter^\I$ to $\priest^\I$,
provided $\renter^\I\neq\domain$, thus violating also $\priest \ISA \cleric$.
Alternatively, we can drop from $\renter^\I$ an individual $o$ that is also in $\priest^\I$,
provided $\priest^\I\neq\eset$, thus violating also $\cleric \ISA \renter$.

Therefore, if $\J$ violates $\priest \ISA \renter$ and has minimal distance to $\I$
with respect to any of the four distances, we have
\begin{itemize}
\item $\dist^a_\setin(\I, \J) = \set{\priest(o)}$ or \\[.5ex]
      $\dist^a_\setin(\I, \J) = \set{\renter(o)}$, for some $o\in\domain$;
\item $\dist^a_\cardin(\I, \J) = 1$;
\item $\dist^s_\setin(\I, \J) = \set{\priest}$ or $\dist^s_\setin(\I, \J) = \set{\renter}$;
\item $\dist^s_\cardin(\I, \J) = 1$.
\end{itemize}
Note that with respect to the symbol-based distances, minimal distance is also kept by
adding more than one element to $\priest$ or dropping more than one element from $\renter$.

We conclude that
\begin{inparaenum}[\itshape (1)]
\item any $\J\in\Mod(\priest \ISA \renter)$ with minimal distance to $\I$
      either satisfies $\priest \ISA \cleric$ or $\cleric \ISA \renter$,
\item if $\renter^\I\neq\domain$, then there is a $\J\in\Mod(\priest \ISA \renter)$  with minimal distance to $\I$
      such that $\J$ violates $\priest \ISA \cleric$, and
\item if $\priest^\I\neq\eset$, then there is a $\J\in\Mod(\priest \ISA \renter)$  with minimal distance to $\I$
      such that $\J$ violates $\cleric \ISA \renter$.
\end{inparaenum}
Thus,
\[
\Mod(\K)\cup \bigcup_{\I\in\Mod(\K)\;} \argmin_{\J \in \Mod(\priest \ISA \renter)}\dist(\I, \J)
\]
satisfies the conditions of Theorem~\ref{lem:no-disjunction-dllite},
which implies the claim for local semantics.

We next consider global semantics.
As we have seen, the minimal distance between some $\I\in\Mod(\K)$ and some $\J\in \Mod(\priest \ISA \renter)$
is a set of cardinality one, or the number 1.
Moreover, for each such $\I$ there exist corresponding interpretations $\J$ with that minimal distance.
If follows that for our example, contraction under a local semantics and its global counterpart coincide,
that is,
$\contr{\K}{\N_T}{\glob^y_z} = \contr{\K}{\N_T}{\loc^y_z}$.
Thus, inexpressibility of contraction w.r.t.\ global semantics follows from the
inexpressibility of contraction w.r.t.\ local semantics.
\end{proof}

Observe that with a similar argument one can show that the expansion operator $\mdalop$ of
Qi and Du~\cite{QiDu09}
(and its stratified extension $\sdalop$),
is not expressible in \dlfr.
This operator is a variant of $\glob^\sbs_\cardin$
where
in Equation~\eqref{eq:global-evolution}
one considers only
models $\J\in\Mod(\N)$ that satisfy
$A^{\J}\neq\emptyset$ for every $A$ occurring in $\K\cup\N$.
The modification does not affect the inexpressibility,
which can again be shown using
Example~\ref{ex:rev-operator-inexpr}.
We also note that $\mdalop$ was developed for KB expansion
with empty ABoxes and the inexpressibility comes
from the non-empty ABox.

%Regarding ABox evolution,
%	we can show the inexpressibility for $\glob^\ats_\setin$ semantics only.
%
%\begin{theorem}
%\dlfr is not closed under TBox expansion
%	and TBox contraction
%	w.r.t.\ $\glob^\ats_\setin$.
%Moreover, this holds already
%	when both initial KB and new information are written in \dlf.
%\end{theorem}

As we showed above, \dllite is closed neither under expansion nor under
contraction.
We investigate now whether the situation changes when we restrict evolution to
affect only the ABox level of KBs.

%%%%%%%%%%%%%%%%%%%%%%%%%%%%%%%%%%%%%%%%%%%%%%%%%%%%%%%%%%%%%%%%%%
\subsubsection{ABox Evolution}
%%%%%%%%%%%%%%%%%%%%%%%%%%%%%%%%%%%%%%%%%%%%%%%%%%%%%%%%%%%%%%%%%%

We start with an example illustrating why ABox expansion w.r.t.\
 $\loc^\ats_\setin$ and $\loc^\ats_\cardin$ is not expressible in \dlfr.

%%%%%%% ex: Local MBAs on atoms are not expressible in \dllite
\input{examples/for_section_31/ex-local-mba-need-disjunction}
%%%%%%%%%%%%%%%%%%%%%%%%%%%%%%%%%%%%%%%%%%%%%%%%%%%%%%%%%%%%%%%%%%

%\begin{theorem}
%  \label{th:local-mba-tbox-unexpressible-dllite}
%  \dllite is not closed under TBox expansion and TBox contraction
%  w.r.t.\
%  $\loc^\sbs_\setin$, $\loc^\sbs_\cardin$,
%  $\loc^\ats_\setin$ and $\loc^\ats_\cardin$.
%\added{Moreover, these hold already when both initial KB and new information
%	are written in \dlcore.}
%\end{theorem}

Next, we develop this example further so that it fits into all four
local semantics and $\glob^\ats$.

%\begin{restatable}{theorem}{ABoxInexpressibility}
\begin{theorem}
  \label{thm:abox-mba-expansion-inexpressible-dllite}
  \dlfr is not closed under ABox expansion
  for $\glob^\ats_\setin$ and $\loc^y_z$, where
  $y \in \set{\sbs, \ats}$, and $z \in \set{\setin, \cardin}$.
Moreover, for local semantics this holds already
        when the initial KB is written in \dlcore,
and for $\glob^\ats_\setin$
        when the initial KB is written in \dlf.
In all five cases, it is sufficient that
        the new information consists of a single ABox axiom.
\end{theorem}
%\end{restatable}

\begin{proof}
The inexpressibility of ABox expansion w.r.t.\ $\loc^\ats_\setin$
        has been shown in Example~\ref{ex:local-mbas-are-bad}.

%	inexpressibility of ABox expansion under $\loc^\ats_\cardin$
%	and ABox contraction under $\loc^\ats_\setin$ %and $\loc^\ats_\setin$
%	can be shown similarly.

We turn now to expansion under $\loc^\ats_\cardin$.
We consider the following fragment of our running example:
\begin{align*}
  \T &= \{~
  \begin{array}[t]{@{}l@{}}
    \empwife\ISA\wife,\\
    \wife \ISA \exists \hashusband,\quad
    \exists \hashusband \ISA \wife,\\
    \priest \ISA \lnot \exists\hashusband^{-} ~\},
  \end{array}\\
  \A &= \{~
  \begin{array}[t]{@{}l@{}}
    \empwife(\mary),\quad \hashusband(\mary,\john),\\
    \priest(\adam),\quad \priest(\bob) ~\}.
  \end{array}\\
  \N &= \set{~\priest(\john)~},
\end{align*}
and $\K = \T\cup\A$.

Let $\I$ be an arbitrary model of $\K$ and $\J$ a model of $\N$.
Clearly,
$\set{\priest(\john),\,\hashusband(\mary,\john)}\subset\I\ominus\J$.
However, depending on $\I$, the symmetric difference $\I\ominus\J$
may contain further atoms.
We distinguish between three main cases.
\begin{compactenum}[1.]
\item In the first case,
      Mary had more than one husband in $\I$.
      Then a minimally different $\J\in\Mod(\N)$ is one where
      she is divorced from John, but stays married to the other husbands.
      Consequently,
      $\I\ominus\J$ contains no atoms other than the ones listed above,
      and the minimal distance between $\I$ and any $\J$ is
      $\card{\I\ominus\J} = 2$.
\item In the second case, John was Mary's only husband in $\I$ and
      there was at least one individual other than John, say Sam, that was not a priest.
      Then a minimally different $\J\in\Mod(\N)$ is one where
      Mary is divorced from John and marries such a Sam.
      Consequently, also $\hashusband(\mary,\sam) \in \I\ominus\J$,
      and the minimal distance between $\I$ and any $\J$ is
      $\card{\I\ominus\J} = 3$.
      Note that for a $\J$ where Mary does not marry again, both
      atoms $\wife(\mary)$ and $\empwife(\mary)$ have to be dropped
      so that $\card{\I\ominus\J} = 4$, which is not minimal for $\I$.
\item In the third case, John was Mary's only husband in $\I$ and
      all individuals other than John were priests.
      Now, as in the previous case, for a $\J$ where Mary does not marry again,
      we have $\card{\I\ominus\J}=4$.
      Similarly, if Mary marries an individual $o\neq\john$ that was a priest,
      then also $\set{\hashusband(\mary,o),\, \priest(o)}\subset\I\ominus\J$,
      so that $\card{\I\ominus\J}=4$.
\end{compactenum}

We observe that in all three cases, including the subcases, one of Adam and Bob
remains a priest in $\J$.
In addition, for an~$\I$ in the third case,
it is possible that in a minimally different~$\J$, Mary is married to one of Adam or Bob,
hence, there is a~$\J$ such that $\J\not\models\priest(\adam)$ and there is a $\J$ such
that $\J\not\models\priest(\bob)$.
Again, we are in the case of Theorem~\ref{lem:no-disjunction-dllite},
which proves that ABox expansion w.r.t.\ $\loc^\ats_\cardin$ is not expressible in \dlfr.

To show the inexpressibility of symbol-based local semantics,
we modify the KB $\K = \T\cup\A$ introduced at the beginning of the proof,
defining
% \begin{align*}
%   \T' &= \T \cup \{\parbox[t]{19em}{%
%                    $P_0\ISA\priest,\,
%                     P_1\ISA\priest,\,
%                     P_2\ISA\priest\}$
%          }   \\
%   \A' &= \A \cup \set{P_1(\adam),\,P_2(\bob)}
% \end{align*}
% defining
\begin{align*}
  \T' &=\T
  \begin{array}[t]{@{}l}
    {}\setminus \set{\empwife\ISA\wife}\\
    {}\cup \set{ %{}\{\parbox[t]{19em}{%
     P_0\ISA\priest,\,
     P_1\ISA\priest,\,
     P_2\ISA\priest }
  \end{array}\\
  \A' &= \A
  \begin{array}[t]{@{}l}
    \setminus \set{\empwife(\mary)} \\
    {}\cup \set{P_1(\adam),\,P_2(\bob)}
  \end{array}
\end{align*}
that is
\begin{align*}
  \T' &= \{~
  \begin{array}[t]{@{}l}
    \wife \ISA \exists \hashusband,\\
    \exists \hashusband \ISA \wife,\\
    \priest \ISA \lnot \exists \hashusband^{-},\\
    P_0\ISA\priest,\
    P_1\ISA\priest,\
    P_2\ISA\priest
    ~\},
  \end{array}\\
  \A' &= \{~
  \begin{array}[t]{@{}l}
    \hashusband(\mary,\john),\\
    \priest(\adam),\
    P_1(\adam),\\
    \priest(\bob),\
    P_2(\bob)
    ~\},
  \end{array}
\end{align*}
and $\K' = \T'\cup \A'$.
The new information $\N$ is as before.
We want to show specifically that
$\upd{\K'}{\N}{}$ is not expressible in \dlfr both w.r.t.\ $\loc^\sbs_\cardin$
and w.r.t.\ $\loc^\sbs_\cardin$.
We consider an aribtrary $\I\in\Mod(\K')$.

By a case analysis that is similar to the one in the proof for $\loc^\ats_\cardin$,
one can show that every $\J$ with minimal distance to~$\I$ satisfies at least one
of the assertions $\priest(\adam)$ and $\priest(\bob)$.
Intuitively, the reason for this is that $\priest(\adam)$ cannot be removed from $\I$ without
removing $P_1(\adam)$,
and $\priest(\bob)$ cannot be removed without removing $P_2(\bob)$.
Therefore, removing both atoms $\priest(\adam)$ and $\priest(\bob)$ leads to a distance between~$\I$ and $\J$ that
involves two additional symbols, namely~$P_1$ \emph{and}~$P_2$,
instead of only one additional symbol, namely either $P_1$ \emph{or}~$P_2$, involved in removing one of the two atoms.

Next, we exhibit models of $\upd{\K'}{\N}{}$ that falsify one of the assertions $\priest(\adam)$ and $\priest(\bob)$.
To this end, we consider a specific model $\I'$ of $\K'$.
Let
\begin{align*}
  \I' ={} &\{~
  \begin{array}[t]{@{}l}
    \wife(\mary),~ \hashusband(\mary,\john),\\
    \priest(\adam),~ P_1(\adam),\\
    \priest(\bob),~ P_2(\bob) ~\} ~\cup{}
  \end{array} \\
  & \{~\priest(o),\, P_0(o) \mid
  \begin{array}[t]{@{}l}
    o\in\domain,
    o\notin \{\adam,\bob,\john\}~\}.
  \end{array}
\end{align*}
One readily verifies that this is indeed a model of $\K'$.
We now check what the models $\J$ of $\N$ with minimal symbol-based distance to $\I'$ look like.
Clearly, $\I$ and $\J$ differ in that $\J$ misses the atom
$\hashusband(\mary,\john)$, but comprises the atom $\priest(\john)$.
Hence, these two atoms are always elements of $\I\ominus\J$.
Among the three cases we distinguished when analysing the example for $\loc^\ats_\cardin$,
the first two do not occur here, since
Mary has only John as a husband and
every individual, except John, is a priest.
Therefore, the following situations are possible in such a~$\J$:
\begin{compactenum}[1.]
%\begin{enumerate}
\item Mary does not remarry, which means that also the atom,
$\wife(\mary)$ is in $\I\ominus\J$.
Thus, in this case the set of symbols occurring in $\I\ominus\J$ is
$\set{\hashusband,\, \priest,\, \wife}$.

\item Mary marries someone other than Adam, Bob, or John, say Sam.
Then, $\I\ominus\J$ contains also the three atoms
$\{\priest(\sam),$ $P_0(\sam),$ $\hashusband(\mary,\sam)\}$
and %\linebreak
the set of symbols
occurring in $\I\ominus\J$ is %\linebreak
$\{\hashusband,\,\priest,\, P_0\}$.

\item Mary marries Adam.
Then, $\I\ominus\J$ contains also the three atoms
$\{\priest(\adam),$ $P_1(\adam),$ $\hashusband(\mary,\adam)\}$
and the set of symbols occurring in $\I\ominus\J$ is
$\set{\hashusband,\, \priest,\, P_1}$.

\item Similarly, if Mary marries Bob,
the set of symbols occurring in $\I\ominus\J$ is
$\set{\hashusband,\, \priest,\, P_2}$.
%\end{enumerate}
\end{compactenum}

Clearly, these four sets of symbols are minimal with respect to
set inclusion, and since they all consist of three elements,
they are also minimal with respect to cardinality.
Two of them falsify one of the assertions $\priest(\adam)$ and
$\priest(\bob)$. Together with the earlier observation that
all $\J\in\upd{\K'}{\N}{}$ satisfy $\priest(\adam)\lor\priest(\bob)$,
this allows us to apply Theorem~\ref{lem:no-disjunction-dllite} and
conclude the inexpressibility w.r.t.\ both
$\glob^\sbs_\setin$ and~$\glob^\sbs_\cardin$.

The inexpressibility of ABox expansion w.r.t.\ $\glob^\ats_\setin$
        can be shown similarly to the case of $\loc^\ats_\setin$,
        but we need to add the assertion $\FUNCT{\hashusband}$
        to the TBox.
In this way we ensure that in every model $\I$ of $\K$,
Mary has only John as husband.
To satisfy the assertion $\wife(\mary)$ she has to remarry,
and the three options for obtaining a model $\J$ of $\N$,
\begin{inparaenum}[\itshape (1)]
\item marrying a non-priest,
\item marrying an anonymous priest, and
\item marrying one of Adam or Bob,
\end{inparaenum}
all lead to differences $\I\ominus\J$ that are minimal
with regard to set inclusion.
Again, application of Theorem~\ref{lem:no-disjunction-dllite}
yields  the claim.
\end{proof}

%%%%%%%%%%%%%%%%%%%%%%%%%%%%%%%%% Table Inexpressibility %%%%%%%%%%%%%%%%%%%%%%%%%%%%%%%%%%%%%%%%%
\newcommand{\DLCORE}{\mbox{\small\textit{DL-Lite}${}_{\mathit{core}}$}}
\newcommand{\DLF}{\mbox{\small\textit{DL-Lite}${}_{\F}$}}

\begin{table}[t]
\centering
\begin{tabular}{|l|c|c|c|c|}
\hline
  &
\multicolumn{2}{|c|}{Expansion} &
\multicolumn{2}{|c|}{Contraction}\\
\cline{2-5}
  & TBox & ABox & TBox & ABox \\
\hline
$\loc^\ats_\setin$  & \DLCORE & \DLCORE & \DLCORE & \DLCORE \\ \hline
$\loc^\ats_\cardin$ & \DLCORE & \DLCORE & \DLCORE & \DLCORE \\ \hline
$\loc^\sbs_\setin$  & \DLCORE & \DLCORE & \DLCORE & \DLCORE \\ \hline
$\loc^\sbs_\cardin$ & \DLCORE & \DLCORE & \DLCORE & \DLCORE \\ \hline
$\glob^\ats_\setin$  & \DLCORE & \DLF  & \DLCORE & \DLF \\ \hline
$\glob^\ats_\cardin$ & \DLCORE & ?     & \DLCORE & ?    \\ \hline
$\glob^\sbs_\setin$  & \DLCORE & ?     & \DLCORE & ?    \\ \hline
$\glob^\sbs_\cardin$ & \DLCORE & ?     & \DLCORE & ?    \\ \hline
\end{tabular}
\caption[Inexpressibility of evolution under MBS in \dllite]%
{Inexpressibility of KB evolution in \dlfr.
Each cell shows
the smallest logic of the \dllite family in which evolution instances have been
exhibited that are not expressible in \dlfr.}
\label{table:inexpressibility}
\end{table}
%%%%%%%%%%%%%%%%%%%%%%%%%%%%%%%%%%%%%%%%%%%%%%%%%%%%%%%%%%%%%%%%%%%%%%%%%%%%%%%%%%%%%%%%%%%%%%%%%%

An anologous result % to Theorem~\ref{thm:abox-mba-expansion-inexpressible-dllite}
holds for ABox contraction.

%\begin{restatable}{theorem}{ABoxInexpressibility}
\begin{theorem}
  \label{thm:abox-mba-contraction-inexpressible-dllite}
  \dlfr is not closed under ABox contraction
  for $\glob^\ats_\setin$ and $\loc^y_z$, where
  $y \in \set{\sbs, \ats}$, and $z \in \set{\setin, \cardin}$.
Moreover, for local semantics this holds already
        when the initial KB is written in \dlcore
and for $\glob^\ats_\setin$
        when the initial KB is written in \dlf.
In all five cases, it is sufficient that
        the new information consists of a single ABox axiom.
\end{theorem}
%\end{restatable}

\begin{proof}
The proof of Theorem~\ref{thm:abox-mba-expansion-inexpressible-dllite}
can be almost literally adopted, if the original KB stays the same
and the information to be contracted is the assertion $\hashusband(\mary,\john)$.

Then, instead of concentrating on the models $\J$ of $\priest(\john)$
that are minimally different from models $\I$ of $\K$,
we consider the set of models of $\K$, augmented by
interpretations~$\J$ that falsify $\hashusband(\mary,\john)$ and
are minimally different from models $\I$ of $\K$.
We find that, for the semantics in question,
each interpretation in the set considered satsifies
$\priest(\adam)\lor\priest(\bob)$,
while there is also a $\J_a$ that falsifies $\priest(\adam)$
and a $\J_b$ that falsifies $\priest(\bob)$.
As before, Theorem~\ref{lem:no-disjunction-dllite}
yields  the claim.
\end{proof}

In Table~\ref{table:inexpressibility}
we summarize our findings about the inexpressibility of KB evolution in \dlfr.
The (in)expressibility of both, ABox expansion and contraction w.r.t.\
$\glob^\ats_\cardin$, $\glob^\sbs_\setin$, and $\glob^\sbs_\cardin$
in \dlfr
remains open problems.

%De Giacomo et al.~\cite{DLPR06} considered ABox evolution
%with protected TBox w.r.t.\ $\loc^\ats_\setin$ semantics.
%They presented an algorithm to compute \dlfr KBs that represent
%$\upd{\K}{\N_A}{}$ for \dlfr KBs $\K$ and $\N_A$.
%As a consequence of Theorem~\ref{th:local-mba-abox-unexpressible-dllite},
%their algorithm is not complete.

%A strange effect of expansion under $\loc^\ats_\cardin$ semantics is
%that new information may \quotes{erase} completely the previous KB.
%
%
%%%%%%%%%%%%%% th: local MB seman on card erases all
%\begin{proposition}
%	\label{th:local-card-erases}
%Let $\K$ be a KB with at least one finite model and
%let $\N$ be a satisfiable KB such that all its models are infinite.
%Then under $\loc^\ats_\cardin$ we have that $\upd{\K}{\N}{} = \N$.
%\end{proposition}
%
%Since the \dllite logics without role functionality
%have the finite model property,
%that is, every satisfiable KB in these logics has a finite model, the above
%situation cannot occur for them.
%At the same time, in every \dllite logic with role functionality there are KBs
%all of whose models are  infinite and such an erasure can take
%place.

%%%%%%%%%%%%%%%%%%%%%%%%%%%%%%%%%%%%%%%%%%%%%%%%%%%%%%%%%%%%
% Discussion
%%%%%%%%%%%%%%%%%%%%%%%%%%%%%%%%%%%%%%%%%%%%%%%%%%%%%%%%%%%%
\subsection{Conceptual Problems of MBAs}
%\dima{Revise the discussion.}
We now discuss conceptual problems with all the local semantics.
Recall Example~\ref{ex:local-mbas-are-bad}
for local MBAs $\loc^a_\setin$ and $\loc^\ats_\cardin$.
We note two problems.
First, the divorce of Mary from John had a strange effect on
the priests Bob and Adam.
The semantics questions their celibacy
and we have to drop the information that they are priests.
This is counter-intuitive, since Mary and her divorce have nothing
to do with any of these priests.
Actually, the semantics also erases from the KB assertions about
all other people belonging to concepts whose instances are not married,
since potentially each of them is Mary's new husband.
Second, a harmless clarification added to the TBox, namely
that ministers are in fact clerics, strangely affects the whole class of
clerics.
The semantics of evolution \quotes{requires} one to allow marriages for clerics.
This appears also strange, because intuitively
the clarification on ministers does not contradict
by any means the celibacy of clerics.

Also the four global MBAs have conceptual problems
that were exhibited in Example~\ref{ex:rev-operator-inexpr}.
The restriction on rent subsidies that
cuts the payments for wives introduces a counterintuitive choice for employed wives.
Under the symbol-based global semantics, they must either
\emph{collectively} get rid of their husbands or
\emph{collectively} lose the subsidy.
Under atom-based semantics the choice is an individual one.

Summing up on both global and local MBAs,
  they focus on minimal change of \emph{models} of KBs
  and, hence, introduce choices that cannot be captured
  in \dllite, which owes its good computational properties to the absence of disjunction.
This mismatch with regard to the structural properties of KBs
  leads to counterintuitive and undesired results,
  like inexpressibility in \dllite
  and erasure of the entire KB.
Therefore, we claim that these semantics are not suitable for the evolution of
  \dllite KBs and now study evolution according to formula-based approaches.

%% file: examples/for_section_31/ex-rev_operators_inexpressible.tex
%!TEX root = ../../main.tex

\begin{example}
	\label{ex:rev-operator-inexpr}
Consider the KB $\KEX$ of our running example and assume that the new
information $\N_T = \set{\wife \ISA \lnot\renter}$ arrived.
We explore expansion
	w.r.t.\ the semantics $\glob^\sbs_\cardin$,
	which counts for how many symbols the interpretation changes.

Consider three assertions,
(derived) from $\K$,
that are essential for this example:
$\empwife \ISA \wife$, $\empwife \ISA \renter$, and $\empwife(\mary)$.
One easily verifies that the minimum of $\dist^\sbs_\cardin(\I,\J)$
for $\I\in\Mod(\K)$ and $\J\in\Mod(\N_T)$ is 1,
since, intuitively, we can turn a model of $\K$ into a model of $\N_T$
by dropping $\mary$ either from $\wife$ or from $\renter$.
Let $\J \in \upd{\K}{\N_T}{\glob^\sbs_\cardin}$.
Then there exists $\I\in\Mod(\K)$ such that $\dist^\sbs_\cardin(\I, \J) = 1$.
Hence, there is only one symbol $S\in \set{\empwife, \wife, \renter}$
whose interpretation has changed from $\I$ to $\J$, that is $S^\I
\neq S^\J$.
Observe that $S$ cannot be $\empwife$.
Otherwise, $\wife$ and $\renter$ would be interpreted identically
under $\I$ and $\J$,
and $\wife$ and $\renter$ would not be disjoint under $\J$,
since $\mary$ is an instance of both,
thus contradicting $\N_T$.
Now, assume that $\wife$ has not changed.
Then $\J\models \empwife\ISA\wife$,
since this held already for $\I$.
However, $\J\not\models \empwife\ISA\renter$,
since $\mary\in \empwife^\J$, but $\mary \notin \renter^\J$,
due to the disjointness of $\wife$ and $\renter$ with respect to $\J$.
Similarly, if we assume that $\renter$ has not changed,
it follows that $\J\models \empwife\ISA\renter$,
but $\J\not\models \empwife\ISA\wife$.
By Theorem~\ref{lem:no-disjunction-dllite}
we conclude that $\upd{\K}{\N_T}{\glob^\sbs_\cardin}$ is not expressible in \dlfr.
%
%Analogously, one can show inexpressibility
%	results for TBox expansion and contraction for all global semantics.
%
\end{example}

%% file: examples/for_section_31/ex-local-mba-need-disjunction.tex
%!TEX root = ../../main.tex

\begin{example}
\label{ex:local-mbas-are-bad}

%%%%%%% ABox Evolution %%%%%%%%%%%%%%%%%%%%%%%%%%%%%%%%%%%%%%%%%%%

We turn again to our KB $\KEX$ and consider the scenario where
we are informed that John is now a priest,
formally $\N_A=\set{\priest(\john)}$.
The TBox assertions essential for this example are
$\empwife\ISA \wife$,
$\wife \ISA \exists \hashusband$,
$\exists \hashusband \ISA \wife$, and
$\priest \ISA \lnot \exists \hashusband^-$,
while the essential ABox assertions are
$\empwife(\mary)$,
$\hashusband(\mary,\john)$,
$\priest(\adam)$, and
$\priest(\bob)$.
Note also that every model of $\KEX$ contains the atom $\wife(\mary)$.
We show the inexpressibility of evolution w.r.t.\ $\loc_\setin^\ats$
using Theorem~\ref{lem:no-disjunction-dllite}.

Under $\loc^\ats_\setin$,
in every $\J\in\upd{\K}{\N_A}{}$ one of four situations holds:
\begin{enumerate}%[\it (i)]
\item
  \label{it:mary-no-husb}
      Mary is not a wife, that is, $\J\not\models
           \wife(\mary)$,
      and both Adam and Bob are priests,
	    that is,
	    $\J \models \priest(\adam)$ and $\J\models\priest(\bob)$.
	  Hence, $\J \models \priest(\adam) \lor \priest(\bob)$.
\item Mary has a husband, who is not John, say Sam.
      Due to minimality of change, both Adam and Bob are still priests,
      as in Case~\ref{it:mary-no-husb},
      and again $\J \models \priest(\adam) \lor \priest(\bob)$.
\item
	  \label{it:mary-priest-husb}
	  Mary is married to Adam, while Bob, due to mininality of change, is still a priest.
      That is, $\J \models \priest(\adam) \lor \priest(\bob)$.
	  Moreover, the new husband cannot stay priest any longer
	  and
	  $\J \not\models \priest(\adam)$.
\item
	 \label{it:mary-another-priest-husb}
      Mary is married to Bob and Adam remains a priest.
	  Analogously to Case~\ref{it:mary-priest-husb},
	  we have $\J \models \priest(\adam) \lor \priest(\bob)$ and
	  $\J \not\models \priest(\bob)$.
\end{enumerate}
%%
% Note that it is not the case that Mary is married to both
% Adam and Bob, since such a model $\J$ is not
% minimally different from any model of the original KB.
In each situation we are in the conditions of Theorem~\ref{lem:no-disjunction-dllite}
and therefore $\upd{\K}{\N_A}{}$ is not expressible in \dlfr.
\end{example}

%%% Local Variables:
%%% mode: latex
%%% TeX-master: "../../main.tex"
%%% End:

%% file: 51-formula-based_approaches.tex
%!TEX root = main.tex

\section{Formula-based Approaches to KB Evolution}
\label{sec:FormulaBasedApproaches}
%%%%%%%%%%%%%%%%%%%%%%%%%%%%%%%%%%%%%%%%%%%%%%%%%%%%%%%%%%%%%%%%%%%%%%

Under formula-based approaches, the objects of change are sets of formulae.
% Before starting a discussion of classical approaches and their adaptation to
%       the case of DL,
%       we need to introduce the following notion.
We recall that without loss of generality
        we can consider only closed KBs,
        that is, if $\K \models \alpha$ for some \dlfr assertion $\alpha$,
        then $\alpha \in \K$.

%We say that a KB $\K = \T \cup \A$ is \emph{closed}
%	if $\T = \cl(\T)$ and $\A = \tcl(\A)$.
%Since for any \dlfr KB $\K$ the closures $\cl(\T)$ and $\tcl(\A)$
%	are finite and moreover can be computed in polynomial time,
%	without loss of generality we can consider only closed KBs.

\subsection{Classical Formula-based Approaches}
Given a closed KB $\K$ and new knowledge $\N$,
        a natural way to define the result of expansion seems to choose
        a maximal subset $\K_m$ of $\K$ such that $\K_m \cup \N$ is coherent and
        to define $\upd{\K}{\N}{}$ as $\K_m \cup \N$.
However, a problem here is that in general such a $\K_m$ is not unique.

Let  $\M_e(\K,\N)$ be the set of all such maximal $\K_m$.
In the past, several approaches
to combine all elements of $\M_e(\K,\N)$ into one set of formulae,
which is then added to $\N$, have been proposed~\cite{EiGo92,Wins90}.
The two main ones are known as
\emph{Cross-Product}, or \CrossP\ for short, and
\emph{When In Doubt Throw It Out}, or \WIDTIO\ for short.
The corresponding sets
$\K_{\CrossP}$ and
$\K_{\WIDTIO}$ are defined as follows:
%
%\begin{align*}
% %\label{eq:semantics}
%       \upd{\K}{\N}{\CrossP} &=
%               \bigset{ \bigvee_{\K_m \in \M_e(\K,\N)}
%               (\bigwedge_{\phi \in \K_m \cup \N} \phi) },\\
% %
%       \upd{\K}{\N}{\WIDTIO} &=
%               \bigcap_{\K_m \in \M_e(\K,\N)} \K_m \cup \N.
% \end{align*}
\begin{align*}
        \upd{\K}{\N}{\CrossP} &=
           \N \cup
               \bigset{\bigvee_{\K_m \in \M_e(\K,\N)}
                (\bigwedge_{\phi \in \K_m} \phi) } \notag, \\
        \upd{\K}{\N}{\WIDTIO} &=
             \N \cup
                \Bigl(\bigcap_{\K_m \in \M_e(\K,\N)} \K_m \Bigr).
% \label{eq:semantics}
\end{align*}
In \CrossP\ one adds to $\N$ the disjunction of all $\K_{m}$,
viewing each $\K_{m}$ as the conjunction of its assertions,
while in \WIDTIO\ one adds to $\N$ those formulas present in all $\K_m$.
In terms of models,
every model of $\K_{\WIDTIO}$ is also a model of $\K_{\CrossP}$,
whose models in turn are exactly the interpretations satisfying
\emph{some} of the $\K_m$.

We can naturally extend this approach to the case of contraction.
Indeed, let $\K_m$ be a maximal subset of $\K$
        such that $\K_m \not\models \alpha$ for each $\alpha \in \N$
        and let $\M_c(\K, \N)$ be the set of all such maximal $\K_m$.
Then we can define contraction under \CrossP\ and \WIDTIO\ as follows:
\begin{align*}
        \contr{\K}{\N}{\CrossP} &=
                \bigset{ \bigvee_{\K_m \in \M_c(\K,\N)}(\bigwedge_{\phi \in \K_m} \phi) },\\
        \contr{\K}{\N}{\WIDTIO} &=
                \bigcap_{\K_m \in \M_c(\K,\N)} \K_m.
\end{align*}

Next,
        we show that these semantics satisfy the evolution postulates
        defined in Section~\ref{sec:postulates}.

\begin{proposition}
\label{prop:classical-FBAs-satisfy-postulates}
Expansion (resp., contraction) of a \dlfr KB under operator
        $\upd{}{}{\mathit{X}}$ (resp.$\contr{}{}{\mathit{X}}$),
        where $\mathit{X} \in \set{\CrossP, \WIDTIO}$,
        satisfies {\EP{1}}--\,{\EP{5}} (resp. {\CP{1}}--\,{\CP{3}} and \CP{5}).
However, contraction under both \CrossP\ and \WIDTIO\
        does not satisfy~\CP{4}.
\end{proposition}

\begin{proof}
The claim for \EP{1}, \EP{2}, \EP{5}, \CP{1}, \CP{2}, and \CP{5} follows directly
        from the definitions of the operators.
\EP{3} (resp., \CP{3}) follows from the observation that
        if $\K \models \N$,
        (resp., if $\K \not\models \alpha$ for each $\alpha \in \N$),
        then $\M_e(\K, \N) = \set{\K}$
        (resp., $\M_c(\K, \N) = \set{\K}$).
Finally, \EP{4} follows from the following observation:
        \begin{itemize}
        \item[\CrossP:] Assume that $\J$ is a model of $(\upd{\K}{\N_1}{\CrossP}) \cup \N_2$.
        Then $\J \models \N_1$, $\J \models \N_2$, and $\J \models \K'_m$
                for some $\K'_m \in \M_e(\K, \N_1)$.
        But in this case we have that $\K'_m \cup \N_1 \cup \N_2$ is satisfiable
                and therefore $\K'_m \in \M_e(\K, \N_1 \cup \N_2)$,
                which shows the claim.

        \item[\WIDTIO:] First observe that for each $\K''_m \in \M_e(\K, \N_1 \cup \N_2)$
                there exists $\K'_m \in \M_e(\K, \N_1)$ such that $\K''_m \incl \K'_m$.
        Assume that $\J$ is a model
                of $(\upd{\K}{\N_1}{\WIDTIO}) \cup \N_2$.
        Then $\J \models \N_1$, $\J \models \N_2$, and $\J \models \alpha$
                for each $\alpha \in \bigcap_{\K'_m \in \M_e(\K, \N_1)} \K'_m$.
        Due to the observation above, we have that
        $\bigcap_{\K''_m \in \M_e(\K, \N_1 \cup \N_2)} \K''_m \incl
        \bigcap_{\K'_m \in \M_e(\K, \N_1)} \K'_m$, which shows the claim.
        \end{itemize}
To see that contraction under both \CrossP\ and \WIDTIO\
        does not satisfy \CP{4}, consider the following example.
Let $\K$ consist of a TBox $\set{A \ISA B}$ and an ABox $\set{A(a)}$, and
        let $\N$ consist of an assertion $B(a)$.
It is easy to see that $\M_c(\K, \N) = \set{\K_m^1, \K_m^2}$,
        where $\K_m^1 = \set{A(a)}$ and $\K_m^2 = \set{A \ISA B)}$.
Then observe that the interpretation $\J = \set{B(a)}$ is a model of
        $(\contr{\K}{\N}{\CrossP}) \cup \N$ since it is a model of\linebreak $\K_m^2 \cup \N$,
        and it is a model of $(\contr{\K}{\N}{\WIDTIO}) \cup \N$
        since $\contr{\K}{\N}{\WIDTIO} = \eset$.
This concludes the proof.
\end{proof}

Intuitively,
        contraction under the two operators does not satisfy~\CP{4},
        since we restrict ourselves to \dlfr,
        and therefore, when getting rid of the information in $\N$,
        we are not able to be too precise and
        delete only what is really required, but we have to delete too much information.

Next,
        we observe some built-in shortcomings of the two semantics.
Consider the following example.

%%%%%%%%%%%%% Example: why WIDTIO and Cross-Product semantics are bad
\input{examples/for_section_32/ex-widtio_cr_pr}

%%%%%%%%%%%%%%%%%%%%%%%%%%%%%%%%%%%%%%%%%%%%%%%%%%%%%%%%%%%%%%%%%%%%%%

Intuitively, \CrossP\ does not lose information,
but the price to pay is that the resulting KB can be exponentially larger than the original KB,
since there can exist exponentially many $\K_m$.
Indeed, consider a KB that contains for each $i\in\{1,\ldots,n\}$, three concepts $A_i$, $B_i$, $C_i$ and the two inclusion assertions
$A_i \sqsubseteq B_i$ and $B_i \sqsubseteq C_i$;
and the new information that states that $A_i$ and $C_i$ are disjoint for all $i$.
Then, there are $2^n$ many maximal coherent subsets $\K_m$.
In addition, as Example~\ref{ex:formula-based-semantics} shows, even if $\K$ is a \dlfr KB,
the result may not be representable in \dlfr any more
since it requires disjunction.
This effect is also present if the new knowledge involves only ABox assertions.

\WIDTIO, on the other extreme, is expressible in \dllite.
However, it can lose many assertions, which may be more than one is prepared
to tolerate.
Even, if one deems this loss acceptable, one has to cope with the fact
that it is computationally complex to decide whether an assertion belongs to
$\upd{\K}{\N}{\WIDTIO}$.
% To make this formal, we define the \emph{WIDTIO membership problem} as the one,
% given a KB $\K$ and new assertions $\N$,
% to decide whether the corresponding set $\K_{\WIDTIO}$
% contains $\alpha$.
This problem is already difficult if our KBs are TBoxes
that are specified in the simplest variant of \dllite.
We note that the following theorem can be seen as a sharpening of a result
about \WIDTIO\ for propositional Horn theories in \cite{EiGo92},
obtained with a different reduction than ours.

%%%%%%%%%%%%%%%%%%%%%%%%%%%%%%%%%%%%%%%%%%%%%%%%%%%%%%%%%%%%%%%%%%%%%%%%%%%
\begin{restatable}{theorem}{widtioComplexity}
\label{thm:widtio-conpcomplete}
For a \dlfr KB $\K$ and new information~$\N$,
        deciding whether an assertion is in
%	$\upd{\T_1}{\N_1}{\WIDTIO}$ (resp., $\contr{\T_2}{\N_2}{\WIDTIO}$)
        $\upd{\K}{\N}{\WIDTIO}$ % (resp., $\contr{\T}{\N}{\WIDTIO}$)
        is \conptime-complete.
Moreover,
        hardness holds already for \dlcore KBs with empty ABoxes.
\end{restatable}

\begin{proof}
The membership in \conptime is straightforward:
to check that an assertion $\phi$ is not in $\upd{\K}{\N}{\WIDTIO}$,
guess a $\K_m$ from $\M_e(\K, \N)$
and verify that $\K_m\cup\N \not\models \phi$.
To see that this is in fact a non-deterministic polynomial time procedure,
note that a subset $\K'$ of $\K$ is an element of $\M_e(\K, \N)$
if for any formula $\gamma\in\K\setminus\K'$,
we have that $\K'\cup\N\cup\set{\gamma}$ is not coherent.
This can be verified in polynomial time for \dllite KBs.

\begin{figure*}[t!]
\centering
\includegraphics[width=1\textwidth]{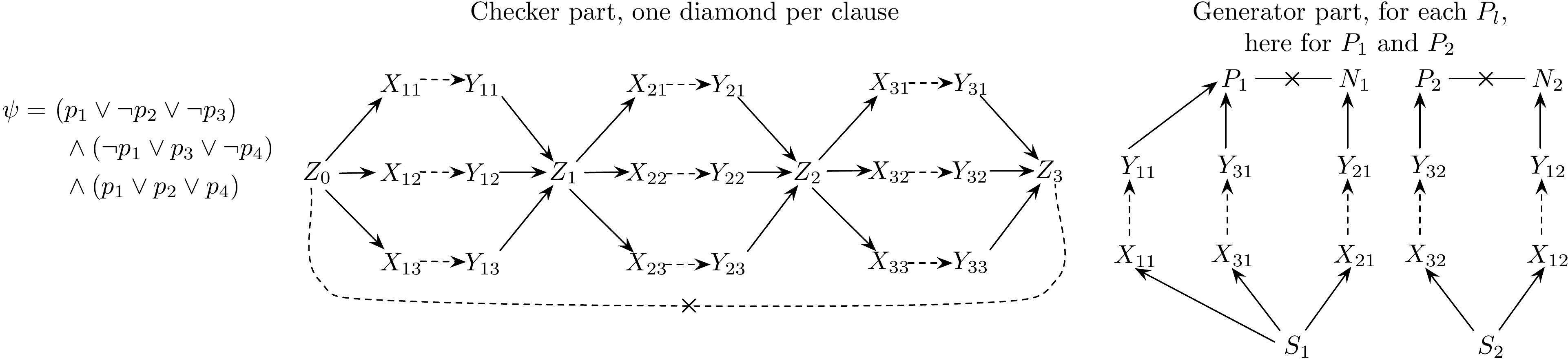}
\caption{Illustration of the $\mathsf{3SAT}$ reduction.}
\label{fig:reduction}
\end{figure*}

\medskip
That the expansion problem is \conptime-hard is shown by a reduction of  $\mathsf{3SAT}$, which is illustrated in Figure~\ref{fig:reduction}.
Let $\psi$ be a 3-CNF formula.
Our plan is to construct KBs $\K^\psi$ and $\N^\psi$, both consisting only of
inclusion assertions.
We single out one assertion $\phi$ of $\K^\psi$ such that $\psi$ is unsatisfiable if and only if
$\upd{\K}{\N}{\WIDTIO} \models \phi$.

Let $p_1,\ldots,p_r$ be the propositional variables occurring in $\psi$.
Without loss of generality, we can assume that each $p_l$ occurs both positively and negatively in $\psi$.
Suppose $\psi$ is a conjunction $\psi_1\land\cdots\land\psi_n$ of $n$ clauses.
Each clause $\psi_i$ is a disjunction of three literals
$L_{i1} \lor L_{i2} \lor L_{i3}$,
where either $L_{ij}=p_i$ for some variable~$p_i$, in which case we say that $L_{ij}$ is positive,
or $L_{ij}=\lnot p_i$, in which case we say that $L_{ij}$ is negative.

The KB $\K^\psi$ models the clauses and their literals by a set of concept inclusion assertions.
For each literal $L_{ij}$ we introduce two concepts $X_{ij}$ and $Y_{ij}$,
together with the assertion
\[
  X_{ij}\ISA Y_{ij},\qquad\text{for } i\in\{1,\ldots,n\},\ j\in\{1,2,3\}.
\]
The KB $\K^\psi$ contains these inclusions and the disjointness axiom
\[
   \phi \ = \ Z_0 \ISA \lnot Z_n.
\]

The new KB $\N^\psi$ consists of two parts, one that
models the possible truth values of the literals,
and a second that
models the logical connections of the literals.

To model the values
assigned to the literals by a truth value assignment,
we introduce for each propositional variable $p_l$ three concepts
$S_l$, $P_l$, and $N_l$.
We insert into $\N^\psi$ the inclusion
\[
   S_l \ISA X_{ij}\quad\mbox{whenever $p_l$ is the variable of $L_{ij}$.}
\]
We connect the corresponding concepts $Y_{ij}$ to either $P_l$ or $N_l$,
depending on whether $p_l$ occurs positively or negatively in $L_{ij}$.
More precisely, we add to $\N^\psi$ the inclusion
\begin{align*}
  Y_{ij} \ISA P_l, &\quad\text{if } L_{ij} = p_l,\\
  Y_{ij} \ISA N_l, &\quad\text{if } L_{ij} = \lnot p_l.
\end{align*}
Finally, we add to $\N^\psi$ the disjointness axiom $P_l\ISA\lnot N_l$.

The intuition behind the reduction becomes clear if we view each inclusion
axiom as an arc in a directed graph, whose nodes are the concepts.
By construction, since $p_l$ occurs both positively and negatively in $\psi$,
there are is a path in $\K^\psi\cup\N^\psi$ from $S_l$ to $P_l$ and another
one from $S_l$ to $N_l$.
Since in any model of $\K^\psi \cup \N^\psi$,
the concepts $P_l$ and $N_l$ are disjoint,
the concept $S_l$, which is contained in both,
is interpreted as the empty set, which makes the KB incoherent.
This can only be prevented by dropping either all paths from $S_l$ to $N_l$ or
all paths from $S_l$ to $P_l$.
Keeping in a maximal coherent subset $\K_m\incl\K^\psi$
all the paths from $S_l$ to~$P_l$, corresponds to
assigning to $p_l$ the value $\mathtt{true}$.
Keeping only the paths from $S_l$ to $P_l$, corresponds to
assigning to $p_l$ the value $\mathtt{false}$.

To model the logic of the clauses,
we introduce into $\N^\psi$ six inclusion axioms per clause.
To this end, we use, in addition to~$Z_0$ and $Z_n$, another $n-1$ concepts
$Z_1,\ldots,Z_{n-1}$.
Then, the six inclusions for the $i$-th clause are
\[
  \begin{array}{r@{~}c@{~}l}
    Z_{i-1} &\ISA& X_{ij} \\
    Y_{ij} &\ISA& Z_i
  \end{array}
  \qquad\text{for } j\in\{1,2,3\}.
\]

Under the graph view of the KB $\K^\psi\cup\N^\psi$,
one can walk from $Z_{i-1}$ to $Z_{i}$ only along three possible paths,
passing one of the arcs $X_{ij}\ISA  Y_{ij}$ corresponding to the literals $L_{ij}$, $j\in\{1,2,3\}$.
This models the disjunction of the three literals appearing in the $i$-th clause $\psi_i$.
To walk from $Z_0$ to $Z_n$, one has to take all the $n$ steps,
from $Z_{i-1}$ to $Z_i$, for $i\in\{1,\ldots,n\}$.
This models the conjunction of the $n$ clauses in $\psi$.
A path from $Z_0$ to $Z_n$ forces~$Z_0$ to be a subset of $Z_n$ in every model of the path.
Together with the disjointness axiom $\phi = (Z_0\ISA \lnot Z_n)$,
this implies that~$Z_0$ is empty, which is not possible, if we want our KB to be coherent.

We are now in a position to show that $\phi$ does not follow from $\upd{\K^\psi}{\N^\psi}{\WIDTIO}$
if and only if $\psi$ is satisfiable.
To this end, assume that $\psi$ is satisfiable and let $\alpha$ be a satisfying assignment.
Let $\K'\subseteq\K^\psi$ contain $X_{ij} \ISA Y_{ij}$
if and only if $\alpha(L_{ij}) = \mathtt{true}$.
Suppose, $\alpha$ satisfies the $j$-th literal of the $i$-th clause, $L_{ij}$.
Then $\K'\cup\N^\psi$ contains a path from $Z_{i-1}$ to $Z_i$, passing through $X_{ij}$ and $Y_{ij}$.
Since, by assumption, $\alpha$ satisfies every clause in $\psi$,
the KB $\K'\cup\N^\psi$ contains a path from $Z_{0}$ to $Z_n$.
As seen above, adding $\phi$ to $\K'\subseteq\N^\psi$ would lead to an incoherent KB.
Thus, with $\K'$ we have exhibited an element of $\M_e(\K^\psi,\N^\psi)$ that does not contain $\phi$,
so that $\phi$ is not in the intersection of the elements of $\M_e(\K^\psi,\N^\psi)$
and thereofore does not follow from $\upd{\K^\psi}{\N^\psi}{\WIDTIO}$.

Next, assume that $\psi$ is unsatisfiable, and
let $\K_m$ be a maximal subset of $\K^\psi$ such that $\N\cup\K_m$ is coherent.
Let $\alpha$ be the assignment such that $\alpha(p_l)=\mathtt{true}$ if
$(X_{ij}\ISA Y_{ij})\in\K_m$ for some positive literal $L_{ij}=p_l$,
and $\alpha(p_l)=\mathtt{false}$ otherwise.
This assignment, like all assignments, by assumption falsifies $\psi$
and in particular falsifies one clause, say the $i$-th one.
Then all literals of that clause are falsified by $\alpha$.

Consider a literal of that clause, say $L_{ij}$.
We make a case analysis as to wheter $L_{ij}$ is a positive or a negative literal.
Suppose $L_{ij}$ is positive, say $L_{ij} = p_l$.
Then $\alpha(p_l) = \mathtt{false}$,
which means that the condition for $\alpha(p_l)$ being $\mathtt{true}$ true does not hold.
Then $\K_m$ contains no inclusion corresponding to a positive $p_l$-literal.
In particular, the inclusion $X_{ij}\ISA Y_{ij}$ for $L_{ij}$ is not in $\K_m$.
Suppose $L_{ij}$ is positive, say $L_{ij} = \lnot p_l$.
Then $\alpha(p_l) = \mathtt{true}$.
By definition of $\alpha$, some inclusion corresponding to a positive $p_l$-literal is present in $\K_m$.
Hence, no arc for a negative $p_l$-literal is in $\K_m$, because otherwise $S_l$ would be incoherent.
Therefore, the inclusion $X_{ij}\ISA Y_{ij}$ corresponding to $L_{ij}$ is not in $\K_m$.

In summary, we have seen that there is no path from $Z_{i-1}$ to $Z_i$ in $\K_m\cup\N^\psi$.
Consequently, $Z_0\ISA Z_n$ does not follow from $\K_m\cup\N^\phi$,
so that $\phi = (Z_0\ISA\lnot Z_n)$ is in $\K_m$, due to the maximality of $\K_m$.
Since $\K_m$ was arbitrary, $\upd{\K^\psi}{\N^\psi}{\WIDTIO}\models\phi$.

This shows that $\upd{\K^\psi}{\N^\psi}{\WIDTIO}\models\phi$ if and only if $\psi$ is unsatisfiable,
which completes the proof.
\end{proof}
%%%%%%%%%%%%%%%%%%%%%%%%%%%%%%%%%%%%%%%%%%%%%%%%%%%%%%%%%%%%%%%%%%%%%%%%%%%

Thus, both \CrossP\ and \WIDTIO\ semantics are computationally problematic,
        even for languages such as \dlfr,
        where the closure of a KB is always finite.
Therefore,
        we conclude that neither \CrossP\ nor \WIDTIO\ is proper for practical solutions.
In the following, we introduce a semantics
        that can help to overcome the issue of intractability.

%% file: examples/for_section_32/ex-widtio_cr_pr.tex
%!TEX root = ../../main.tex

\begin{example}
\label{ex:formula-based-semantics}
We consider again our running example.
Suppose we obtain the new information that priests no longer obtain rental subsidies.
This can be captured by the set of TBox assertions $\N_T = \set{\priest \ISA \lnot\renter}$.
We now incorporate this information into our KB, under both \CrossP\ and
\WIDTIO\ semantics.
Clearly, $\KEX \cup \N_T$ is not coherent
and to resolve the conflict one can drop either
$\priest \ISA \cleric$ or $\cleric \ISA \renter$.
Hence, $\M_e(\KEX, \N_T) = \set{\K_m^{(1)}, \K_m^{(2)}}$,
  where $\K_m^{(1)} = \KEX \setminus \set{\priest \ISA \cleric}$, and
        $\K_m^{(2)} = \KEX \setminus \set{\cleric \ISA \renter}$.
Consequently, the results of evolving $\K$ with respect to $\N_T$
under the two semantics are
\begin{align*}
  \upd{\KEX}{\N_T}{\CrossP}
    &~=~ \N_T \cup
    \bigl(
    \begin{array}[t]{@{}l}
      (\K \setminus \set{\priest \ISA \cleric})\\[.5ex]
      {}\lor \left(\K \setminus \set{\cleric \ISA \renter}\right) \bigr)
    \end{array}
    % \label{eqn:pc-evolution}
    \\
  \upd{\KEX}{\N_T}{\WIDTIO}
    &~=~ \N_T \cup \left(\K_m^{(1)} \cap \K_m^{(2)}\right) \notag\\
    &~=~ (\N_T \cup \KEX) \setminus \{
    \begin{array}[t]{@{}l}
      \priest \ISA \cleric,\\[.5ex]
      \cleric \ISA \renter \},
    \end{array}
    %\notag
\end{align*}
where in the first formula
% ~\eqref{eqn:pc-evolution}
we have combined DL notation with first order logic notation.
\end{example}

%% file: 52-bold_semantics.tex
%!TEX root = main.tex

\subsection{Bold Semantics}
\label{sec:Bold-Semantics}
%%%%%%%%%%%%%%%%%%%%%%%%%%%%%%%%%%%%%%%%%%%%%%%%%%%%%%%%%%%%%%%%%%%%%%
As we have seen above,
        the classical approaches \CrossP\ and \WIDTIO\
        may pose practical challenges in the case of \dlfr.
Indeed, the former one is inexpressible in \dlfr,
        since it requires disjunction,
        but even for more expressive languages where \CrossP\ is expressible,
        the resulting KB, after a series of evolutions,
        is going to be very complicated and overloaded.
The latter semantics is always expressible in \dlfr;
        computing the result under it, however, is computationally hard even for \dlcore.
Besides, the \WIDTIO\ semantics tends to delete too much information.

Recall that both \CrossP\ and \WIDTIO\ semantics were proposed to combine all elements
        of $\M_e(\K, \N)$ or $\M_c(\K, \N)$ into a single KB.
We propose another way to deal with the problem of multiple maximal KBs:
        instead of combining the different $\K_m$, we suggest to \emph{choose one of them}.
We call this semantics \emph{bold}.
More formally, we say that
        $\K'$ is a result of expansion (resp., contraction) of $\K$ w.r.t.\ $\N$
        if $\K' \equiv \K_m \cup \N$ for some $\K_m \in \M_e(\K, \N)$
        (resp., $\K' \equiv \K_m$ for some $\K_m \in \M_c(\K, \N)$).
An obvious drawback of this approach is that the choice of $\K_m$ is not deterministic.
Consider the following example.

%%%%%%%%%%%%% Example: bold semantics
\input{examples/ex-bold_semantics}

%%%%%%%%%%%%%%%%%%%%%%%%%%%%%%%%%%%%%

\begin{algorithm2e}[t!]%[H]
\caption{$\BE(\K, \N)$}
\label{alg:bold-expansion}

\KwIn{closed KBs $\K$ and $\N$}
\KwOut{KB $\K'$}

\BlankLine

$\K' := \N$;\quad $\S := \K$;

\Repeat
{$\S = \eset$}
{choose some $\phi \in \S$;\\
 $\S := \S \setminus \set{\phi}$;\\
 \If
 {$\K' \cup \set{\phi}$ is coherent}
 {$\K' := \K' \cup \set{\phi}$}
}

\Return $\K'$;
\end{algorithm2e}
We continue now with a check of how bold semantics satisfies the evolution postulates.
% We continue now with that the bold semantics satisfies the evolution postulates.
But first observe that the postulates \EP{4}, \EP{5}, and \CP{5}
        do not make much sense in the context of bold semantics,
        due to its non-determinism.
Therefore,
        we first propose an alternative version of those postulates
        to take into consideration the non-determinism of bold semantics:
\begin{quote}
        \EP{4B}: For each $\K''_m \in \M_e(\K, \N_{1e} \cup \N_{2e})$,
        there exists a $\K'_m \in \M_e(\K, \N_{1e})$ such that $\K'_m \models \K''_m$.\\[2ex]
        \EP{5B}: Expansion should not depend on the syntactical representation of knowledge,
        that is,
        if $\K_1 \equiv \K_2$ and $\N_{1e} \equiv \N_{2e}$,
        then $\M_e(\K_1, \N_{1e}) \equiv \M_e(\K_2, \N_{2e})$.\\[2ex]
        \CP{5B}: Contraction should not depend on the syntactical representation of knowledge,
        that is,
        if $\K_1 \equiv \K_2$ and $\N_{1c} \equiv \N_{2c}$,
        then $\M_c(\K_1, \N_{1c}) \equiv \M_c(\K_2, \N_{2c})$.
\end{quote}

\begin{proposition}
\label{prop:bold-semantics-satisfies-postulates}
For the evolution of \dlfr KBs
under bold semantics the following holds:
\begin{compactitem}
\item Expansion satisfies {\EP{1}}--\,{\EP{3}}, \EP{4B}, and \EP{5B};
\item Contraction satisfies  {\CP{1}}--\,{\CP{3}} and \CP{5B}, but not \CP{4}.
\end{compactitem}
\end{proposition}

\begin{proof}
The claim for \EP{1}, \EP{2}, \EP{5B}, \CP{1}, \CP{2}, and \CP{5B} follows directly
        from the definitions of the operators.
The claim for \EP{3}, \EP{4B}, and \CP{3}
        and the fact that contraction does not satisfy \CP{4}
        can be proved similarly
        to the corresponding claims in Proposition~\ref{prop:classical-FBAs-satisfy-postulates}.
\end{proof}

Which of the two possible results in Example~\ref{ex:bold-semantics} should one choose?
We claim that the choice is domain-dependent
        and, consequently, it should be made by a user/domain expert.
In our particular example,
        the right choice seems to pick the second KB
        since it is possible that clerics do not receive rent subsidies,
        while the first option where priests stop being clerics does not make
        sense.

\subsubsection{Bold Semantics without User Preferences}
%%%%%%%%%%%%%%%%%%%%%%%%%%%%%%%%%%%%%%%%%%%%%%%%%%%%%%
Consider the case when the user does not have any preferences
        and any of the possible results of evolution would be satisfactory.
In this case,
        choosing an arbitrary $\K_m$ has the advantage
        that the result of evolution can be computed in polynomial time.
Algorithms~\ref{alg:bold-expansion} and~\ref{alg:bold-contraction}
        can be used to compute the result of expansion or contraction, respectively,
        in a non-deterministic manner.

%%%%%%%%%%%%%%% alg: cpmputation of Bold semantics
%\input{algorithms/alg-bold-semantics.tex}
%%%%%%%%%%%%%%%%%%%%%%%%%%%%%%%%%%%%%%%%%%%%%%%%%%

\begin{algorithm2e}[t!]%[H]
\caption{$\BC(\K, \N)$}
\label{alg:bold-contraction}

\KwIn{closed KBs $\K$ and $\N$}
\KwOut{KB $\K'$}

\BlankLine

$\K' := \eset$;\quad $\S := \K$;

\Repeat
{$\S = \eset$}
{choose some $\phi \in \S$;\\
 $\S := \S \setminus \set{\phi}$;\\
 \If
 {$\K' \cup \set{\phi} \not\models \alpha$ for each $\alpha \in \N$}
 {$\K' := \K' \cup \set{\phi}$}
}

\Return $\K'$;
\end{algorithm2e}

%%%%%%%%%%%%%%%%
\begin{theorem}
  \label{th:semantics-correctness}
For \dlfr KBs $\K$ and $\N$,
        the algorithms $\BE$ and $\BC$
        run in time polynomial in $|\K \cup \N|$
        and compute a bold expansion and a bold contraction of $\K$ by $\N$, respectively.
\end{theorem}
%%%%%%%%%%%%%%%%

\begin{proof}
The fact that the algorithms compute the results of expansion and contraction, respectively,
        is obvious.
To prove polynomiality,
        observe that the algorithms loop as many times as there are assertions in $\K$.
The crucial steps are the coherence steps for \BE
        and the entailment checks for \BC.
It is well known, however, that in \dlfr these checks can be done in polynomial time.
\end{proof}

%given a KB $\K$ and new knowledge $\N$,
%returns a set $\K_m\incl\cl(\K)$ that is a maximal compatible
%set of assertions for $\K$ and $\N$.
%The algorithm loops as many times as there are assertions in $\cl(\K)$.
%The number of such assertions is at most quadratic in
%the number of constants, atomic concepts, and roles.
%The crucial step is the check for coherence,
%which is performed once per loop.
%If this test is polynomial in the size of the input
%then the entire runtime of the algorithm is polynomial.
%For \dlfr TBoxes $\T$, coherence can be checked in time
%quadratic in the number of assertions in the TBox,
%that is, $O(|\T|^2)$.
%Satisfiability of an ABox $\A$ with respect to $\T$ can be checked in time
%$O(|\T|^2\times|\A|)$, where $|\A|$ is the number of assertions of $\A$.
%The $O(|\T|^2)$ complexity can be shown by
%reduction to satisfiability of sets of propositional Horn clauses
%(see~\cite{ZCKN10} for details).

\subsubsection{Bold Semantics with User Preferences}
%%%%%%%%%%%%%%%%%%%%%%%%%%%%%%%%%%%%%%%%%%%%%%%%%%%%%%
We have seen that computing an arbitrary $\K_m$
        has the great advantage that
        evolution can be computed in polynomial time.
However, its non-determinism is a disadvantage.
Clearly, we can \emph{avoid} nondeterminism if we impose a linear order
on the assertions over the signature of $\K$, and let \BE and \BC choose them in this order.
The question how to define such an order is again application-dependent
        and is out of the scope of our work.

% An approach to \emph{reduce} nondeterminism is to single out a subset of $\M(\K,\N)$
% whose elements one considers preferable and to try to steer $\KU$ in such
% a way that it returns an element of that preferred subset.
% In the following we show that, unfortunately,
% for natural definitions of preferred sets $\K_m$
% computation becomes difficult.

% We first assume that all assertions of $\K$ are equally valuable.
% Then it may be desirable to compute a $\K_m$ with maximal \emph{cardinality.}
% (Recall that our algorithm is only guaranteed to compute a $\K_m$ that is
% maximal wrt set inclusion.)
% It turns out
% (by reductions from the Independent Set problem),
% that under this requirement computation is hard,
% even for $\K$ and $\N$ that consist
% only of TBox or only of ABox assertions,
% except when both, $\K$ and $\N$, are ABoxes,
% in which case no conflicts can arise.

A natural question
that requires further investigation
        is whether there exist preferences as to
        which $\K_m$ to use for constructing the result of evolution
        such that they are generic enough and
        can be implemented without breaking tractability.
% In the remainder of the section we study some of such preferences.

One may also wonder whether it is possible
        to efficiently compute a $\K_m$ with maximal \emph{cardinality}.
        Recall that our algorithm is only guaranteed to compute a $\K_m$ that is
        maximal w.r.t.\ set inclusion.
Unfortunately, it turns out
% using various reductions from the Independent Set Problem,
that under this requirement computation is hard,
even if $\K$ is either a TBox or an ABox and $\N$ is a TBox.

%%%%%%%%%%%%%%%%%
\begin{theorem}
        \label{th:max-card-np-complete}
Given \dlfr KBs $\K$ and $\N$ and a subset $\K_0 \incl \K$
        such that $\K_0 \cup \N$ is coherent,
%	\added{(resp., $\K_0 \not\models \alpha$ for each $\alpha \in \N$)},
        deciding whether $\K_0$ has maximal cardinality among the elements
        of $\M_e(\K, \N)$ %\added{(resp., $\M_c(\K, \N)$)}
        is \nptime-complete.
Moreover, %in the case of expansion,
        \nptime-hardness already holds for \dlcore if
\begin{inparaenum}[\itshape (1)]
        \item both $\K$ and $\N$ are TBoxes, or
        \item $\K$ is an ABox and $\N$ is a TBox.
\end{inparaenum}
\end{theorem}

\begin{proof}
This problem is equivalent to the problem of deciding
        whether there exists a subset $\K_1$ of $\K$
        such that $\K_1 \cup \N$ is coherent and $|\K_1| \geq |\K_0| + 1$.
        %\added{(resp., $\K_1 \not\models \alpha$ for each $\alpha \in \N$)}
We prove now that this latter problem is \nptime-complete.
Indeed, the membership in \nptime is obvious:
        guess a subset $\K_1$ of $\K$ of size greater than $|\K_0|$
        and check whether $\K_1 \cup \N$ is coherent,
        %\added{(resp., $\K_1 \not\models \alpha$ for each $\alpha \in \N$)},
        which can be done in polynomial time.
We show hardness by a reduction of the Independent Set Problem
        for graphs
        to the problem of evolution of a \dlfr KB under bold semantics.
Given a graph $G = (V, E)$, a subset $V'$ of $V$ is called independent,
        if for any pair $u$ and $v$ in $V'$ the edge $(u,v)$ is not in $E$.
Deciding whether for a given integer $m \leq |V|$
        an independent set of size $m$ or more exists
        is known to be \nptime-complete.

%\noindent\textbf{Expansion.}
To prove the statement for Case~1, %for expansion,
        we use the following reduction.
The TBox $\T$ consists of the assertions $S \ISA A_i$
        for each $v_i \in \V$,
        and the new information $\N$ consists of the assertions $A_i \ISA \lnot A_j$
        for each $(v_i, v_j) \in E$.
Clearly, a subset $\T_1 = \set{S \ISA A_k \mid k\in\{i_1, \ldots, i_m\}}$
        of $\T$ has the property that $\T_1 \cup \N$ is coherent
        if and only if $\set{v_{i_1}, \ldots, v_{i_m}}$	is an independent set.

To prove the statement for Case~2,
        we use the following reduction.
The ABox $\A$ consists of the membership assertions $A_i(b)$ for each $v_i \in V$,
        and the new information $\N$ is as in the previous case.
Clearly, a subset $\A_1 = \set{A_{i_1}(b), \ldots, A_{i_m}(b)}$
        of $\A$ is such that $\A_1 \cup \N$ is coherent
        if and only if $\set{v_{i_1}, \ldots, v_{i_m}}$	is an independent set.
\end{proof}

In the next section we will see that nondeterminism is not present in
ABox evolution, where the TBox is protected, and that there is always a single maximal
compatible ABox.

%% file: examples/ex-bold_semantics.tex
%!TEX root = ../main.tex

\begin{example}
\label{ex:bold-semantics}
Consider the KB and the new information from
Example~\ref{ex:formula-based-semantics}.
As shown there, $\M(\KEX, \N_T) = \set{\K_m^{(1)}, \K_m^{(2)}}$.
According to bold semantics
  the result of expansion is a KB $\K' = \N \cup \K_m$
  for some $\K_m \in \M(\KEX, \N_T)$.
Thus, the result of expansion is either \
  $\N_T \cup \KEX \setminus \set{\priest \ISA \cleric}$ or
  $\N_T \cup \KEX \setminus \set{\cleric \ISA \renter}$.
\end{example}

%% file: 60-abox_evolution.tex
%!TEX root = main.tex

\section{Formula-based Approaches to ABox Evolution}
\label{sec:AndABoxEvolution}
%%%%%%%%%%%%%%%%%%%%%%%%%%%%%%%%%%%%%%%%%%%%%%%%%%%%%%%%%%%%%%%%%%%%%%

In this section we study ABox evolution under formula-based approaches.
First,
        observe that the classical approaches, \CrossP\ and \WIDTIO,
        can be easily adapted to ABox evolution
        by requiring additionally that $\T$ is a part of $\K_m$.
Note that this additional requirement does not contradict the general definition of $\K_m$.
Indeed,
\begin{itemize}
\item In the case of expansion, since in the case of ABox evolution
        we assume that $\T \cup \N$ is coherent,
        the requirement that $\T \incl \K_m$ does not contradict that $\K_m \cup \N$
        is coherent.

\item In the case of contraction, since a \dlfr TBox does not entail any ABox assertion,
        the requirement that $\T \incl \K_m$ does not contradict
        that $\K_m \not\models \alpha$, for each $\alpha \in \N$.
\end{itemize}
This requirement, however, brings a surprising result:
        it makes a maximal subset $\K_m$ unique.

%%%%%%%%%%%
\begin{restatable}{proposition}{entailmentsInDLLite}
        \label{lem:entailments-in-dl-lite}
Let $\K = \T \cup \A$ be a \dlfr KB.
Then
        \begin{compactitem}
        \item If $\K \models \alpha_1$, where $\alpha_1$ is a \dlfr membership assertion,
%	then there exists $\alpha_2 \in \tcl(\A)$
        then there exists $\alpha_2 \in \A$
        such that $\T \cup \set{\alpha_2} \models \alpha_1$.
        \item If $\K$ is unsatisfiable,
%	then there exist $\alpha_1, \alpha_2 \in \tcl(\A)$
        then there exist $\alpha_1, \alpha_2 \in \A$
        such that $\T \cup \set{\alpha_1, \alpha_2}$ is unsatisfiable.
        \end{compactitem}
\end{restatable}
%%%%%%%%%%%

\input{proofs/60-lem-entailments_in_dl-lite}

Proposition~\ref{lem:entailments-in-dl-lite}
        immediately gives us the following lemma.

%%%%%%%%%%%%%
\begin{lemma}
  \label{lem:ABox-evolution-unique-Km}
Let $\K$ be a \dlfr KB with TBox $\T$ and $\N$ a \dlfr ABox.
Then there exists exactly one element $\K_m$
        in $\M_e(\K,\N)$ (resp., in $\M_c(\K,\N)$)
        such that $\T \incl \K_m$.
\end{lemma}
%%%%%%%%%%%%%

\begin{proof}
Suppose $\K = \T\cup\A$.
Then $\K_m$ is obtained by dropping from $\K$,
for each $\beta\in\N$, all ABox assertions $\alpha\in\A$ such that
$\T\cup \set{\beta,\, \alpha}$ is unsatisfiable.
\end{proof}

%\begin{proof}
%Let $\K = \T \cup \A$.
%First,
%	we consider the case of expansion.
%It follows from the proof of Lemma~A.1 from~\cite{KharlamovZhCjcss2013}
%	that if $\T \cup (\A \cup \N)$ is inconsistent,
%	then there are two membership assertions $f_1$ and $f_2$ in $\A \cup \N$
%	such that $\T \cup \set{f_1, f_2}$ is inconsistent.
%Since $\T \cup \A$ and $\T \cup \N$ are coherent,
%	we conclude that one assertion, say $f_1$, belongs to $\A$ and the other one to $\N$.
%If $\T \incl \K_m$ and $\K_m \cup \N$ is coherent,
%	then $f_1 \notin \K_m$.
%Going through all such pairs of membership assertions,
%	we conclude that $\K_m$ with $\T \incl \K_m$ is uniquely defined.
%
%In the case of contraction,
%	we again turn to the proof of Lemma~A.1 from~\cite{KharlamovZhCjcss2013},
%	which implies that if $\T \cup \A \models f_1$,
%	then there exists a membership assertion $f_2$ in $\A$
%	such that $\T \cup \set{f_2} \models f_1$.
%This immediately implies that $\K_m$ with $\T \incl \K_m$ is uniquely defined.
%\end{proof}

The straightforward consequence of this property
        is that the classical formula-based approaches, \CrossP\ and \WIDTIO,
        and the proposed bold semantics coincide.
Also observe that ABox evolution under bold semantics
        becomes deterministic,
        so we will use the binary operators $\upop_b$ and $\controp_b$
        to designate ABox expansion and contraction, respectively,
        under bold semantics.

\begin{corollary}
\label{cor:three-mbas-coincide}
Let $\K = \T \cup \N$ be a \dlfr KB
and $\N$ a \dlfr ABox.
Then, assuming that $\T \cup \N$ is coherent, ABox expansion
        (resp., ABox contraction) under \CrossP, \WIDTIO, and bold semantics
        coincide.
\end{corollary}

Next we study whether bold semantics
        satisfies the evolution postulates
        in the case of ABox evolution.

\begin{proposition}
\label{prop:bold-semantics-for-TBox-evolution-satisfies-postulates}
For ABox evolution of \dlfr KBs under bold semantics the following holds:
\begin{compactitem}
\item ABox expansion satisfies {\EP{1}}--\,{\EP{5}};
\item ABox contraction satisfies  {\CP{1}}--\,{\CP{3}} and \CP{5},
      but not \CP{4}.
\end{compactitem}
\end{proposition}

\begin{proof}
The claim follows from
        Proposition~\ref{prop:bold-semantics-satisfies-postulates}
        and the observation that in the case
        when $\M_e(\K, \N)$ (resp., $\M_c(\K, \N)$)	is a singleton,
        \EP{\textit{i}B} implies \EP{\textit{i}} (resp., \CP{5B} implies \CP{5}).
The fact that contraction does not satisfy \CP{4}
        can be shown as in
        Proposition~\ref{prop:bold-semantics-satisfies-postulates}.
\end{proof}

In principle, \BE and \BC can be used to compute ABox evolution
        under bold semantics (and also \CrossP\ and \WIDTIO)
        with the only change in Line~1
        that we set $\K' := \T \cup \N$ in \BE
        and $\K' := \T$ in \BC.
Regardless of the order in which the algorithms select the assertions,
        they will always return the same result.
A drawback of the algorithms is that
        they respectively perform a coherence and entailment check during each
        loop iteration.
We exhibit now new algorithms \FE and \FC
        that do not perform those checks;
        instead, they perform checks at the syntax level.

%%%%%%%%%%%%%%%%%%%%%%%%%%%%%%%%%%%%%%%%%%%
\input{algorithms/alg-weeding}

%%%%%%%%%%%%%%%%%%%%%%%%%%%%%%%%%%%%%%%%%%%

We start with the algorithm \FC.
The algorithm (see Algorithm~\ref{alg:fast-contraction})
        works as follows:
        it takes as input a closed \dlfr KB $\T \cup \A$
        and a set of \dlfr ABox assertions $\N$,
        and returns as output an ABox $\A'$
        such that
        \begin{inparaenum}[\it (i)]
        \item $\A' \incl \A$ and
        \item $\T \cup \A' \not\models \alpha$ for each $\alpha \in \N$.
        \end{inparaenum}
Now we show the correctness of the algorithm.

%%%%%%%%%%%%%%%%%%%%%%%%%%%%%%%%%%%%
\begin{theorem}
\label{thm:fast-contraction-correctness}
The algorithm \FC computes an ABox contraction under bold semantics,
        that is,
        $\contr{(T \cup \A)}{\N}{b} = \T \cup \FC(T \cup \A, \N)$,
        and runs in polynomial time.
\end{theorem}
%%%%%%%%%%%%%%%%%%%%%%%%%%%%%%%%%%%%

\begin{proof}
The proof % of the theorem
        is based on the proof of Lemma~\ref{lem:ABox-evolution-unique-Km}.
Let $\A' = \FC(T \cup \A, \N)$ and $\K' = \T \cup \A'$.
We show that $\K' \in \M_c(\K, \N)$.
First,
        we show that $\K' \not\models \alpha$ for each $\alpha \in \N$.
Indeed,
        assume that this is not the case and there is an $\alpha \in \N$
        such that $\K' \models \alpha$.
We know that
        there exists an inclusion assertion $\phi \in \T$
        and a membership assertion $\beta \in \A'$
        such that $\set{\phi, \beta} \models \alpha$.
We have five possible cases:
\begin{itemize}
\item $\alpha = \beta$.
In this case we have that $\beta$ was removed from $\A'$ at Line~3
        during the corresponding loop iteration.

\item $\alpha$ is of the form $B_1(c)$,
        $\beta$ is of the form $B_2(c)$, and
        $\phi$ is of the form $B_2 \ISA B_1$.
But then $\beta$ was removed from $\A'$ at Line~5.

\item $\alpha$ is of the form $\exists R(c)$,
        $\beta$ is of the form $R(c,d)$, and
        $\phi$ does not matter.
In this case we have that $\beta$ was added to $\N$ at Line~8
        and removed from $\A'$ at Line~14.

\item $\alpha$ is of the form $\exists R_1(c)$,
        $\beta$ is of the form $R_2(c,d)$, and
        $\phi$ is of the form $R_2 \ISA R_1$.
In this case we have that $R_1(c,d) \in \A$
        and it was added to $\N$ at line~8,
        and then $\beta$ was removed from $\A'$ at Line~15.

\item $\alpha$ is of the form $R_1(a,b)$,
        $\beta$ is of the form $R_2(a,b)$, and
        $\phi$ is of the form $R_2 \ISA R_1$.
But then $\beta$ was removed from $\A'$ at Line~15.
\end{itemize}
In any case we have a contradiction.

The maximality of $\K'$ follows straightforwardly from the following
observation:
        if a membership assertion $\beta$ is from $\K \setminus \K'$,
        then it was removed from $\A'$ at Line~3, 5, 14, or~15.
Then clearly, $\K' \cup \set{\beta} \models \alpha$
        for some $\alpha \in \N$,
        which shows the maximality of $\K'$ and concludes the proof.
\end{proof}

Now we turn to \FE (see Algorithm~\ref{alg:fast-expansion}).
First,
        the algorithm detects the assertions in $\A$
        that conflict with the new information $\N$
        and stores them in $\confa$.
Then it resolves these conflicts using \FC as a subroutine.
Finally, the algorithm returns the conflict-free part of $\A$ together with $\N$.

%%%%%%%%%%%%%%%%%%%%%%%%%%%%%%%%%%%%%%%%%%%
\input{algorithms/alg-compute-naive-update}

%%%%%%%%%%%%%%%%%%%%%%%%%%%%%%%%%%%%%%%%%%

%%%%%%%%%%%%%%%
\begin{theorem}
\label{thm:fast-expansion-correctness}
The algorithm \FE computes an ABox expansion under bold semantics,
        that is, $\upd{(\T \cup \A)}{\N}{b} = \T \cup \FE(\T \cup \A, \N)$,
        and runs in polynomial time.
\end{theorem}
%%%%%%%%%%%%%%%

\begin{proof}
Let $\K = \T \cup \A$, $\A' = \FC(\K, \confa)$,
        where $\confa$ is as built by the algorithm from Lines~1--13, and
        $\K' = \T \cup \A'$.
We show that $\K' \in \M_e(\K, \N)$.
First,
        we show that $\K' \cup \N$ is consistent.
Indeed,
        assume that this is not the case.
We know~\cite{CDLLR07} that there exists
        a TBox assertion $\phi$ of the form
        $B_1 \ISA \lnot B_2$ or $\FUNCT{R}$
        and a pair of membership assertions $\alpha \in \N$ and $\beta \in \A'$
        such that $\set{\phi, \alpha, \beta}$ is inconsistent.
We have two possible cases:
\begin{itemize}
\item $\phi$ is of the form $B_1 \ISA \lnot B_2$,
        $\alpha$ is of the form $B_1(c)$ (resp., $B_2(c)$), and
        $\beta$ is of the form $B_2(c)$ (resp., $B_1(c)$).
But then we have that $\beta$ was added to $\confa$
        at Line~5 (resp., Line~4) and removed from $\A$ by \FC.

\item $\phi$ is of the form $\FUNCT{R}$,
        $\alpha$ is of the form $R(a,b)$, and
        $\beta$ is of the form $R(a,c)$.
But then we have that $\beta$ was added to $\confa$
        at Line~10 or~11 and removed from $\A$ by \FC.
\end{itemize}
In any case,
        we have a contradiction.

The maximality of $\K'$ follows straightforwardly from the following
observation:
        if a membership assertion $\beta$ is from $\K \setminus \K'$,
        then it was removed from $\A$ by \FC
        and thus, added to $\confa$ at Line~4, 5, 10, or~11.
Then clearly, $\K' \cup \set{\beta} \cup \N$ is inconsistent,
        which shows the maximality of $\K'$ and concludes the proof.
\end{proof}

Note that both Algorithm~\ref{alg:fast-contraction} and
Algorithm~\ref{alg:fast-expansion}
        expect a closed KB $\K$ as input.
The algorithms can be optimised so as to deal with non-closed KBs.
However, this kind of optimisation is outside the scope of our work.

%% file: proofs/60-lem-entailments_in_dl-lite.tex
\begin{proof}
The proposition directly follows from the results in~\cite{CDLLR07}.
\end{proof}

%% file: algorithms/alg-weeding.tex
\begin{algorithm2e}[t!]
\caption{$\FC(\K, \N)$}
\label{alg:fast-contraction}
\SetKwFor{ForEach}{for each}{do}{end}

\KwIn{
 \begin{tabular}[t]{@{}l}
   closed KB $\T \cup \A$;\\
   ABox $\N$
 \end{tabular}}
\KwOut{Abox $\A'$}

\BlankLine

$\A' := \A$\;
\ForEach
{$B_1(c) \in \N$}
{$\A' := \A' \setminus \set{B_1(c)}$\;
 \ForEach
 {$B_2 \ISA B_1 \in \T$}
 {$\A' := \A' \setminus \set{B_2(c)}$\;
  \If {$B_2(c) = \exists R(c)$}
  {\ForEach
   {$R(c,d) \in \A'$}
   {$\N := \N \cup \set{R(c,d)}$}
  }
 }
}

\ForEach
{$R_1(a,b) \in \N$}
{$\A' := \A' \setminus \set{R_1(a,b)}$\;
 \lForEach
 {$R_2 \ISA R_1 \in \T$}
 {$\A' := \A' \setminus \set{R_2(a,b)}$}
}

\Return $\A'$;
\end{algorithm2e}

%%%%%%%%%%%%%%%%%%%%%%%%%%%%%%%%%%%%%%%%

%% file: algorithms/alg-compute-naive-update.tex
\begin{algorithm2e}[t!]
\caption{$\FE(\K, \N)$}
\label{alg:fast-expansion}
\SetKwFor{ForEach}{for each}{do}{end}

\KwIn{\begin{tabular}[t]{@{}l}
   closed KB $\T \cup \A$;\\
   ABox $\N$ s.t.\ $\T \cup \N$ is coherent
 \end{tabular}}
\KwOut{Abox $\A'$}

\BlankLine

$\N := \tcl(\N)$;\quad $\A_0 := \A \cup \N$;\quad $\confa := \eset$;

\ForEach
{$B_1 \ISA \lnot B_2 \in \T$}
{\If
 {$\set{B_1(c),\ B_2(c)} \incl \A_0$}
 {\lIf
  {$B_1(c) \notin \N$}
  {$\confa := \confa \cup \set{B_1(c)}$}
  \lElse
  {$\confa := \confa \cup \set{B_2(c)}$}
 }
}

\ForEach
{$\FUNCT{R} \in \T$}
{\If
 {$\set{R(a,b),\ R(a,c)} \incl \A_0$}
 {\lIf
  {$R(a,b) \notin \N$}
  {$\confa := \confa \cup \set{R(a,b)}$}
  \lElse
  {$\confa := \confa \cup \set{R(a,c)}$}
 }
}

\Return $\FC(\T \cup \A, \confa) \cup \N$;
\end{algorithm2e}

%%%%%%%%%%%%%%%%%%%%%%%%%%%%%%%%%%%%%%%%

%% file: 70-related_work.tex
%!TEX root = main.tex

%%%%%%%%%%%%%%%%%%%%%%%%%%%%%%%%%%%%%%%%%%%%%%%%%%%%%%%%
\section{Related Work}
\label{sec:related-work}
%%%%%%%%%%%%%%%%%%%%%%%%%%%%%%%%%%%%%%%%%%%%%%%%%%%%%%%%

We provide an overview of related work, concentrating mostly on propositional
logic and on Description Logics.

\subsection{Evolution in Propositional Logic KBs}
%%%%%%%%%%%%%%%%%%%%%%%%%%%%%%%%%%%%%%%%%%%%%%%%%%%%%%%%

One of the first systematic studies of knowledge evolution
        that set the foundations of the area has been conducted
        by Alchourr\'{o}n, G\"{a}rdenfors, \andauthor Makinson~\cite{AGM85}.
This work is commonly accepted as the most influential in the field
        of knowledge evolution and belief revision.
The reason is that
        it proposed, on philosophical grounds, a set of \emph{rationality postulates}
        that the operations of \emph{revision} (adding information) and \emph{contraction} (deleting information) must satisfy.
Note that it used the term revision instead of expansion,
        which is used in this paper,
        and,
        in fact, that term is more commonly found in the literature.
The postulates were well accepted by the research community and
        nowadays they are known as \emph{AGM postulates}, named after the three authors who proposed them.

Dalal~\cite{Dala88}
        introduced the \emph{principle of irrelevance of syntax},
        which states that the KB resulting from evolution should not depend on
        the syntax
        (or representation) of the old KB and the new information.
A~number of evolution approaches that meet the AGM postulates
        as well as Dalal's principle
        were proposed in the literature;
        the most well-known are by
        Fagin, Ullman, \andauthor Vardi~\cite{FaginUV:1983pods},
        Borgida~\cite{Borg85},
        Weber~\cite{Weber:1986eds},
        Ginsberg~\cite{Ginsberg:1986ai},
        Dalal~\cite{Dala88},
        Winslett~\cite{Winslett:1988aaai},
        Satoh~\cite{Sato88}, and
        Forbus~\cite{Forb89}.

Winslett~\cite{Winslett:1988pods,Wins90} proposed the classification of
evolution semantics into \emph{model-based} semantics and
\emph{formula-based} semantics, which is the distinction that we have
adopted in this paper.
%
% Under the model-based paradigm, the semantics of a KB evolution is based on
% the individual models of the KB,
%% Thus, MBS are syntax-independent.
%
% while under the formula-based paradigm, one does not examine the models of
% the KB, but rather the axioms in the KB itself.
%% Hence, FBS are syntax-dependent.
The operators from~\cite{FaginUV:1983pods,Ginsberg:1986ai}
        fall into the latter category,
        while the rest of the works cited above fall into the former category.

Katsuno \andauthor Mendelzon~\cite{KatsunoM:1989ijcai}
        gave a model-theoretic characterisation
        of model-based revision semantics that satisfied the AGM postulates.
Keller \andauthor Winslett~\cite{KellerW:1985tse} introduced a taxonomy of
knowledge evolution
        that is orthogonal to the one in~\cite{Wins90}.
They distinguished two types of adding information
        in the context of extended relational databases:
        \emph{change-recording updates} and \emph{knowledge-adding updates}.
Later on Katsuno \andauthor Mendelzon~\cite{KaMe91}
        extended this work to the evolution of KBs,
        referring to change-recording updates as \emph{updates}
        and to knowledge-adding updates as \emph{revision}.
Intuitively, an update brings the KB up to date
        when the real world changes.
The statement ``John got divorced and now he is a priest'' is an example of an
update.
Instead, revision is used when one obtains some new information about a static
world.
For example, we may try to diagnose a disease and we want to incorporate
        into the KB the result of successive tests.
Incorporation of these tests is revision of the old knowledge.
Both update and revision have applications where one is more suitable
than the other.
Moreover, Katsuno \andauthor Mendelzon showed that
        the AGM postulates and the model-theoretic characterisation
        of~\cite{KatsunoM:1989ijcai} are applicable to revision only.
To fill the gap,
        they provided postulates and a model-theoretic characterisation for
        updates~\cite{KaMe91}.
Their model-theoretic characterisation became prevalent
        in the KB evolution and belief revision literature.

% In Section~\ref{sec:mbas},
%       we extend the model-theoretic characterisation of~\cite{KaMe91}
%       from the case of propositional logic to the the first-order logic
%
%\begin{inparaenum}[\it (i)]
%	\item from the case of propositional logic to the first-order logic and
%	\item from the case of simple adding information to the case
%		when some information can be added and another can be deleted at the same time.
%\end{inparaenum}
%
% and study different evolution semantics under this characterisation.

\subsection{Evolution of Description Logic KBs}
%%%%%%%%%%%%%%%%%%%%%%%%%%%%%%%%%%%%%%%%%%%%%%%%%%%%%%%%

Much less is known about the evolution of Description Logic knowledge bases
than about the evolution of propositional logic, and the study of the topic is
rather fragmentary.

Kang \andauthor Lau~\cite{KangL:2004kes} discussed the feasibility of using the concept of belief revision
        as a basis for DL ontology revision.
Flouris, Plexousakis, \andauthor Antoniou~\cite{FlourisPA:2004nmr,FlourisPA:2005iswc}
        generalised the AGM postulates
        in order to apply the rationalities behind the AGM postulates
        to a wider class of logics,
        and determined the necessary and sufficient conditions for a logic
        to support the AGM postulates.
However, none of~\cite{KangL:2004kes,FlourisPA:2004nmr,FlourisPA:2005iswc}
        considered the explicit construction of a revision operator.
Qi, Liu, \andauthor Bell~\cite{QiLB:2006jelia} reformulated the AGM postulates for revision
        and adapted them to deal with disjunctive KBs expressed in the
        well-known DL $\mathcal{ALC}$ .

Later, Qi, et al.~\cite{QiHHJPV08}
        proposed a general revision operator to deal with incoherence.
However, this operator is not fine-grained, in the sense that it removes from a KB
        a whole TBox axiom by an incision function
        as soon as it affects the KB's coherency.

Haase \andauthor Stojanovic~\cite{HaSt05} proposed a formula-based approach for
ontologies in OWL-Lite (which is a DL that is much more expressive than
\dllite), where the removal of inconsistencies between the old and the new
knowledge is strongly syntax-dependent.  Notice instead that our formula-based
semantics are syntax independent.

Liu et al.~\cite{DBLP:journals/ai/LiuLMW11} considered several standard DLs of
the $\mathcal{ALC}$ family~\cite{BCMNP03}, and studied the problem of ABox
updates with empty TBoxes, in the case where the new information consists of
atomic (possibly negated) ABox statements.  They showed that these DLs are not
closed even under simple updates.  However, when the DLs are extended with
nominals and the ``@'' constructor of hybrid logic~\cite{ArecesR:2000aiml}, or,
equivalently, admit nominal and Boolean ABoxes, then updates can be expressed.
They also provided algorithms to compute updated ABoxes for several expressive
DLs and studied the size of the resulting ABoxes. They showed that in general
such ABoxes are exponential in the size of the update and the role-nesting
depth of the original ABox, but that the exponential blowup can be avoided by
considering so-called \emph{projective updates}.  They also consider
conditional updates and how they can be applied to the problem of reasoning
about actions.

The latter problem is also the motivation for Ahmetaj et al.~\cite{ACOS17}, who
study the evolution of extensional data under integrity constraints formulated
in very expressive DLs of the $\mathcal{ALC}$ family, and in \dllite.  The
updates are finite sequences of conditional insertions and deletions, where
complex DL formulas are used to select the (pairs of) nodes for which (node or
arc) labels are added or deleted.  The updates are finite sequences of
conditional insertions and deletions, in which complex DL formulas are used to
select the (pairs of) individuals to insert or remove from atomic
concepts/roles.  The paper studies the complexity of verifying when a sequence
of update operations preserves the integrity constraints, by using a form of
regression that reduces the problem to satisfiability checks over the initial
KB.  \cite{CaOS16} extends the results on verification to the case where the DL
may contain constructs for path-like navigation over the data.

Qi \andauthor Du~\cite{QiDu09} considered a model-based revision operator for
DL terminologies
        (i.e., KBs with empty ABoxes) by adapting Dalal's operator.
They showed that subsumption checking in \dlcore under their revision operator
        is $\mathtt{P}^{\mathtt{NP}[O(\log n)]}$-complete
        and provided a polynomial time algorithm to compute
        the result of revision for a specific class of input KBs.
Observe that with the same argument as the one we used in the proof of
Theorem~\ref{thm:tbox-mbs-contraction-inexpressible-dllite},
one can show that the expansion operator $\mdalop$ of~\cite{QiDu09}
(and its stratified extension $\sdalop$),
is not expressible in \dlfr.
This operator is a variant of $\glob^\sbs_\cardin$,
where in Equation~\eqref{eq:global-evolution}
one considers only models $\J\in\Mod(\N)$ that satisfy
$A^{\J}\neq\emptyset$ for every $A$ occurring in $\K\cup\N$.
The modification does not affect the inexpressibility,
which can again be shown using
Example~\ref{ex:rev-operator-inexpr}.
We also note that $\mdalop$ was developed for KB expansion
with empty ABoxes and the inexpressibility comes
from the non-empty ABox.

De Giacomo et al.~\cite{DLPR09} considered ABox-update and erasure for the DL \dlf.
They considered Winslett's approach
        (originally proposed for relational theories~\cite{Wins90})
        and showed that \dlf is not closed under ABox-level update and erasure.
The results in Section~\ref{sec:mbas} extend these results in the following
directions:
\begin{inparaenum}[\itshape (i)]
        \item we showed new inexpressibility results for many other operators,
including the operator from~\cite{DLPR09}, and
        \item we considered both expansion and contraction at both KB and ABox level.
\end{inparaenum}

Wang, Wang, \andauthor Topor~\cite{WangWT:2010aaai}
        introduced a new semantics for DL KBs and
        adapted to it the MBA.
In contrast to classical model-based semantics,
where evolution is based on manipulation with first-order interpretations,
their approach is based on manipulation of so-called \emph{features},
which are similar to models.
In contrast to models,
        features are always of finite size
        and any DL KB has only finitely many features.
They applied feature-based semantics to $\dllite^\N_{\textit{bool}}$ \cite{ACKZ09},
        and it turned out that the approach suffers
        from the same issues as classical model-based semantics.
        For example,
        DLs are not closed under these semantics
        even for simple evolution settings.
Due to these problems, they addressed approximation of evolution semantics,
        but it turned out to be intractable.
We conjecture that their semantics fits into our framework or Section~\ref{sec:mbas}
        after a suitable extension,
        but our work does not extend their results.
However, observe that the inexpressibility results of~\cite{WangWT:2010aaai}
reaffirm our arguments in Section~\ref{sec:mbas}, where we argued that
model-based approaches suffer from intrinsic expressibility problems.

Lenzerini and Savo~\cite{LenzeriniS:2011dl} considered the
        ``when in doubt throw it out'' (\WIDTIO) approach
        for the case of \textit{DL-Lite}$_{A, \textit{id}}$
        and presented a polynomial time algorithm for computing the evolution of KBs at the instance-level.
Qi et al.~\cite{Qi:2015aaai}
        considered the problem of computing a maximal sound approximation
        of \dlfr KB expansions for two model-based operators.
De Giacomo et al.~\cite{DBLP:conf/semweb/GiacomoORS16} took a different
approach to instance-level formula-based update of \dllite KBs:
given an update specification, they rewrite it into a set of addition and deletion instructions over the ABox, which can be characterized as the result of a first-order query.
This was proved by showing that every update can be reformulated into a Datalog program that generates the set of insertion and deletion instructions to change the ABox while preserving its consistency w.r.t.\ the TBox.
De Giacomo et al.~\cite{DBLP:conf/semweb/GiacomoLOST17} looked at practical
aspects of ontology update management in the context of ontology-based data
access, where ontologies are `connected' to relational data sources via
declarative mappings \cite{XCKL*18}.  In this scenario they study changes or
evolution that affect ontologies and the source data and show how changes can
be computed via non-recursive Datalog.

\subsection{Consistent Query Answering Over Inconsistent KBs}
%%%%%%%%%%%%%%%%%%%%%%%%%%%%%%%%%%%%%%%%%%%%%%%%%%%%%%%%
Knowledge evolution is closely related to consistent
query answering over inconsistent KBs, see e.g.~\cite{DBLP:conf/aaai/Bienvenu12,DBLP:journals/ws/LemboLRRS15,DBLP:conf/ijcai/BienvenuR13,DBLP:conf/aaai/Bienvenu12,DBLP:conf/ijcai/BienvenuBG16},
where the goal is, given a query $Q$ and an inconsistent KB $\K$,
to retrieve `meaningful' answers for $Q$ over $\K$.%
\footnote{Note that since \K is inconsistent it holds that $\K \models Q(\vec
 c)$ for every tuple $\vec c$ of constants with the $\textit{arity}(Q)$ and
 thus every tuple of constants of the appropriate arity is an answer to $Q$
 over $\K$.}
This problem has originally been introduced in the context of databases~\cite{DBLP:conf/pods/ArenasBC99}
and then adapted to KBs.
Meaningful answers are typically defined
using the notion of \emph{repairs}: a KB $\K'$ is a repair of $\K$
if it is consistent and can be obtained by `modifying' $\K$,
e.g., by taking a (set-inclusion maximal) consistent subset of \K (or its deductive closure).
%via a sequence of operations such as deletion and insertion of facts in $\K$ or more generally axioms.
Then, semantics of $Q$ over $\K$ is defined as the intersection of $\textit{ans}(Q,\K')$ over all repairs $\K'$ of $\K$ that are optimal w.r.t.\ some criterion.
% some minimality criterion, e.g., (the deductive closure of) $\K'$ is a set-inclusion maximal subset of (the deductive closure of) $\K$ that is consistent.
Thus, query answering over inconsistent KBs is related to formula-based approaches to evolution, and in particular to \WIDTIO, while to the best of our knowledge no work considers MBAs to KB repair.
Observe that results analogous to our \conptime-completeness of \WIDTIO\
(see Theorem~\ref{thm:widtio-conpcomplete}), which we first reported
in~\cite{CalvaneseKNZ:2010iswc}, have been shown in the context of consistent
query answering after our work has been published, e.g.,
in~\cite{DBLP:conf/ijcai/BienvenuBG16,DBLP:conf/aaai/Bienvenu12,DBLP:journals/ws/LemboLRRS15}.

\subsection{Justification and Pinpointing}
%%%%%%%%%%%%%%%%%%%%%%%%%%%%%%%%%%%%%%%%%%%%%%%%%%%%%%%%
Approaches to knowledge evolution that are often used in
practice, in particular for TBox evolution,
are essentially syntactic~\cite{HaSt05,DBLP:conf/esws/KalyanpurPSG06,DBLP:journals/dke/Jimenez-RuizGHL11}.
Many of them are based on \emph{justification} or \emph{pinpointing}: a
minimal subset of the ontology that entails a given consequence~\cite{DBLP:conf/kr/PenalozaS10,DBLP:journals/ws/KalyanpurPSH05,DBLP:conf/ijcai/SchlobachC03,DBLP:journals/jar/SchlobachHCH07,DBLP:conf/semweb/KalyanpurPHS07}.
For example, to contract $\K$ with an assertion $\phi$ entailed by $\K$,
it suffices to compute all justifications for $\phi$ in $\K$, find a minimal subset $\K_1$ of $\K$ with at least one assertion from each justification,
and take $\K' = \K \setminus \K_1$ as the result of evolution.
This complies with a `syntactical' notion of minimal
change: retracting $\phi$ requires to delete a minimal
set of assertions from \K and hence the structure of $\K$ is maximally
preserved. Moreover, such $\K'$ always exists even for expressive
DLs, and practical algorithms to compute it have been implemented
in ontology development platforms~\cite{DBLP:conf/semweb/KalyanpurPHS07,DBLP:conf/aswc/SuntisrivarapornQJH08}. By removing $\K_1$ from $\K$, however,
we may inadvertently retract consequences of $\K$
other than $\phi$, which are `intended'.
Identifying and recovering such intended consequences is an important issue.
Evolution approaches considered in our work are logic-based rather than
syntactic.
Cuenca Grau et al.~\cite{DBLP:conf/kr/GrauJKZ12} present
a framework to bridge the gap between logic-based and syntactic evolution
approaches.
% that combines formula-based approaches and syntactic approaches.
In particular they propose a new principle of minimal change that has two
dimensions: a structural one ($\K'$ should not change much the structure of
$\K$) and a deductive one (that corresponds to the one we have for formula
based evolution).
Their work is focused on the DLs of the $\mathcal{EL}$ family and
does not consider model based evolution, which is crucial in our study.
Moreover their evaluation algorithm for what they call \emph{finite
 preservation languages} (\dllite is included in this case) corresponds to a
combination of our \BE and \BC algorithms.

\subsection{Diagnosis and Debugging}
%%%%%%%%%%%%%%%%%%%%%%%%%%%%%%%%%%%%%%%%%%%%%%%%%%%%%%%%
In diagnosis and debugging~\cite{DBLP:conf/sum/HorridgePS09,DBLP:journals/ws/KalyanpurPSH05,DBLP:conf/ijcai/SchlobachC03} the goal is to find the KB assertions that cause inconsistency.
This is relevant since, e.g., a formula-based expansion of a KB with new knowledge can lead to its inconsistency and thus debugging techniques can help in finding what causes this inconsistency.
There are attempts to relate these areas and KB evolution, e.g., Ribeiro and
Wassermann~\cite{DBLP:journals/logcom/RibeiroW09}
show how debugging services can be linked to belief revision.  However, further
investigation is required to gain a deeper understanding of the relation.

%% file: 80-conclusion.tex
\section{Conclusions and Future Work}
\label{sec:conclusion}
%%%%%%%%%%%%%%%%%%%%%%%%%%%%%%%%%%%%%%%%%%%%%%%%%%%%%%%%%%%%%%%%%%%%

In this paper we have studied evolution of \dllite KBs, taking into account
both expansion and contraction.  We have considered two main families of
approaches: model-based ones and formula-based ones.
We have singled out and investigated a three-dimensional space of model-based
approaches,
  and have proven that most of them are not appropriate for \dllite,
  due to their counterintuitive behavior
  and the inexpressibility of evolution results.
Thus, we have examined formula-based approaches,
have shown that the classical ones are again inappropriate for \dllite,
and have proposed a novel semantics called \emph{bold}.
We have shown that this semantics can be computed in polynomial time,
  but the result is, in general, non-deterministic.
Then, we have studied ABox evolution under \BS and
  have shown that in this case the result is unique.
We have developed polynomial time algorithms for \dllite KB expansion and
contraction under this semantics, and
alternative optimized variants of the algorithms for ABox evolution.
% We presented a conceptual drawback of ABox evolution under \BS and introduced
% \CS, which repairs the drawback.  For this approach we proved that the
% evolution result is unique and developed a polynomial time algorithm to
% compute it.

The first important conclusion from our work is
that model-based approaches are intrinsically problematic for KB evolution,
even in the case of such a lightweight DL as \dllite.
Indeed, recall that \dllite is not closed under evolution for \emph{any} of
the model-based semantics and thus these semantics are impractical.
As a consequence, one has either
to search for conceptually different semantics
that rely on other principles of `composing'
the output set of models constituting the evolution result,
or one has to develop natural restrictions on how model-based approaches can
`compose' this set.
An alternative approach would be to develop approximation techniques
that allow one to efficiently capture evolution results.

A second important conclusion is that classical formula-based approaches are
too heavyweight from the computational point of view and thus their
practicality is questionable.
On the other hand, the most conceptually simple model-based semantics such as
\BS
can potentially lead to practical evolution algorithms.  However, their
practicality requires further empirical evaluation.
Finally, we have discussed that the classical evolution postulates that were originally developed for propositional theories
are not directly applicable to the case of first-order knowledge
since they are blind to some fundamental properties of such knowledge, such as
coherency.
We have shown how to adapt such postulates to the richer setting considered
here, and have analyzed whether the various model-based and formula-based
semantics satisfy the revised postulates.

We believe that our work opens new avenues for
        research in the area of knowledge evolution,
        which is an important part of knowledge engineering,
        since it shows how to lift approaches to knowledge evolution
        from the propositional to the first-order case.
Moreover, we have presented techniques that allow one to
        prove inexpressibility of model-based evolution,
        and  \conptime-hardness of formula-based evolution.
We believe that these techniques can be relevant
        to knowledge management tasks beyond evolution.

We see several important directions for future work.
First, the problem of expressibility in \dllite is still open for various
model-based evolution semantics (see Table 1).  These settings are all for ABox
expansion and contraction under global model-based semantics.
An important research direction is to apply in practice the ideas we developed and, in particular, to implement an ontology evolution system.
The system can be based on formula-based approaches and implement
Algorithms~1--4 that we proposed.
Such system could also be based on approximations of model-based
semantics, which are out of the scope of this paper, see,
e.g.,~\cite{DBLP:journals/jcss/KharlamovZC13,DLPR09}.
Then, it would be interesting to conduct an empirical evaluation for various
semantics, in order to establish
which semantics give more intuitive results from the users' point of
view, and which ABox evolution approaches are more scalable.
A further direction to investigate is to identify the minimum extensions of
\dllite that would allow it to capture the results of model-based evolution for
\dllite KBs.
For this, one can draw inspiration from the work
in~\cite{DBLP:journals/ai/LiuLMW11}, already discussed in
Section~\ref{sec:related-work}.
Also, it is still unknown what are minimal DLs
that are closed under local model-based evolution, and in general that are well
tailored towards model-based approaches.
Finally, we believe that it is important to develop knowledge evolution
techniques where the user has a much better control over the evolution process.
For this, one can draw inspiration from previous work,
e.g.,
from~\cite{DBLP:conf/kr/GrauJKZ12}, where the authors proposed techniques to  control what syntactic structures of a given KB cannot be changed by the evolution process,
or from~\cite{DBLP:conf/aaai/ZheleznyakovKH17}, where the authors proposed to combine knowledge evolution with models of trust, i.e., the new knowledge in their approach is only partially trusted (note that this scenario inherits the inexpressibility issues of MBAs).